\documentclass[journal]{IEEEtran}

\ifCLASSINFOpdf

\else

\fi

\ifCLASSOPTIONcompsoc
\usepackage[nocompress]{cite}
\else
\usepackage{cite}
\fi
\usepackage{amsmath}
\usepackage{algorithm}
\usepackage{algpseudocode}
\usepackage{array}
\usepackage{subfigure}
\usepackage{multirow}
\usepackage{booktabs}
\usepackage{threeparttable}
\usepackage{stfloats}
\usepackage{amsfonts,amssymb}
\usepackage{color}
\usepackage{bbm}
\usepackage{bm}
\usepackage{float}

\usepackage{hyperref}       
\usepackage{url}            
\usepackage{graphicx}

\usepackage{amsthm}
\usepackage{verbatim}

\newtheorem{theorem}{Theorem}

\theoremstyle{remark}

\begin{document}

\title{CCSL: A Causal Structure Learning Method from Multiple Unknown Environments}

\author{
		Wei Chen,
		Yunjin Wu,
        Ruichu Cai*,~\IEEEmembership{Senior Member,~IEEE},
        Yueguo Chen,
        Zhifeng Hao,~\IEEEmembership{Senior Member,~IEEE}
\thanks{Manuscript received XX; revised XX; accepted XX. Date of publication XX XX, 2020; date of current version XX XX, 2020. This research was supported in part by National Key R\&D Program of China (2021ZD0111501), National Science Fund for Excellent Young Scholars (62122022), Natural Science Foundation of China (61876043, 61976052), the major key project of PCL (PCL2021A12), and China Postdoctoral Science Foundation (2021M690734).  (\emph{Corresponding author: Ruichu Cai.})
}

\thanks{Wei Chen is with the School of Computer, Guangdong University of Technology, Guangzhou 510006, China. \protect e-mail: chenweiDelight@gmail.com.}
\thanks{Yunjin Wu is with the School of Computer, Guangdong University of Technology, Guangzhou 510006, China. \protect e-mail: wwuyunjin@gmail.com.}
\thanks{Ruichu Cai is with the School of Computer, Guangdong University of
Technology, Guangzhou 510006, China, and also with Peng Cheng Laboratory, Shenzhen, China 518066, China. \protect e-mail: cairuichu@gmail.com.}
\thanks{Yueguo Chen is with Key Laboratory of Data Engineering and Knowledge Engineering (MOE), Renmin University of China, Beijing 100872, China. \protect e-mail: chenyueguo@ruc.edu.cn}
\thanks{Zhifeng Hao is with the School of Computer, Guangdong University of Technology, Guangzhou 510330, China and College of Science, Shantou University, Shantou 515063, China. \protect  e-mail: zfhao@gdut.edu.cn.}
\thanks{* Corresponding author.}
}

\markboth{IEEE Transactions on Neural Networks and Learning Systems,~submitted}%
{Wei Chen \MakeLowercase{\textit{et al.}}: CCSL: A Causal Cluster Structure Learning Method from Multiple Unknown Environments}

\maketitle
\begin{abstract}
   Most existing causal structure learning methods assume data collected from one environment and independent and identically distributed (i.i.d.). In some cases, data are collected from different subjects from multiple environments, which provides more information but might make the data non-identical or non-independent distribution. Some previous efforts try to learn causal structure from this type of data in two independent stages, i.e., first discovering i.i.d. groups from non-i.i.d. samples, then learning the causal structures from different groups. This straightforward solution ignores the intrinsic connections between the two stages, that is both the clustering stage and the learning stage should be guided by the same causal mechanism. Towards this end, we propose a unified Causal Cluster Structures Learning (named CCSL) method for causal discovery from non-i.i.d. data. This method simultaneously integrates the following two tasks: 1) clustering samples of the subjects with the same causal mechanism into different groups; 2) learning causal structures from the samples within the group. Specifically, for the former, we provide a Causality-related Chinese Restaurant Process to cluster samples based on the similarity of the causal structure; for the latter, we introduce a variational-inference-based approach to learn the causal structures. Theoretical results provide identification of the causal model and the clustering model under the linear non-Gaussian assumption. Experimental results on both simulated and real-world data further validate the correctness and effectiveness of the proposed method.
\end{abstract}

\section{Introduction}

Causal structure learning is an important method for exploring the data generation mechanism. These data generation mechanism assist interventions and decision-making, which is used in a wide range of fields such as neuroscience \cite{huang2019specific}, bioinformatics \cite{sachs2005causal}, social network analysis \cite{chen2020mining} and so on.

Typical methods for learning causal structure among observed variables include constraint-based method \cite{spirtes2000causation}, score-based method \cite{chickering2002optimal}\cite{huang2018generalized} and functional-based method \cite{shimizu2006linear}\cite{hoyer2008nonlinear}\cite{zhang2009identifiability}. Constraint-based method utilizes (conditional) independence test to determine whether there exists causal relationship between two variables, while score-based method tries to find the causal structure to obtain the best score. Typical functional-based methods include Linear Non-Gaussian Acyclic Model (LiNGAM) \cite{shimizu2006linear}, Additive Noise Model(ANM) \cite{hoyer2008nonlinear} and Post NonLinear model (PNL) \cite{zhang2009identifiability}, which are based on the assumption of data generation process, assuming a linear or nonlinear causal relationship between observations. A number of subsequent approaches are extensions of functional-based method, including time-series data scenarios \cite{hyvarinen2010estimation}, hidden variables case \cite{hoyer2008estimation} \cite{chen2021causal}, and cyclic causal graph \cite{lacerda2008discovering}.
 
However, the majority of existing causal structure learning approaches assume data are collected from same environment and are independent and identical distribution (i.i.d.), which is usually violated in real context where the observed data are collected from multiple unknown environments. Multiple environments provide multiple interventions on a causal mechanism, and arise multiple different causal mechanisms. Different causal mechanisms imply different causal structures, and generate different samples. Consequently, these samples are not independently and identically distributed (non-i.i.d.). An illustration of the data generation process is shown in Figure \ref{fig:gen_process}. When applying the existing methods that assume i.i.d. data to those kind of samples, the dependencies between samples from multiple unknown environments will give rise to spurious causal relationships, making the results of causal structure learning unreliable \cite{ramsey2011meta}\cite{zhang2017causal}.

Recently, many researches tackle non-i.i.d data causal discovery problem using two independent stages: clustering stage and causal structure learning stage \cite{huang2019specific}\cite{wang2020high}\cite{zhang2020unmixing}. That is, first cluster multiple samples into several groups by traditional clustering method, and then learn causal structures from the clustered data. Generally, the traditional clustering method considers the similarity of features, which may be unstable with different environments. On the contrary, the causal relationships among features for each subject is stable and invariant with different environments \cite{zhang2020unmixing}. Because the causal relationships are reflected by the data generation mechanism, they are the inherent nature of each subject. For example, in fMRI data analysis, a few regions of interest (ROIs) interact differently when the subject performs different activities like opening and closing the eyes. If considering all the relationships among ROIs, some ROIs in different activities may be similar, while the causal mechanisms behind different activities are different. Some works pay attention to this problem, and cluster samples based on the learned causal structures. But these methods require prior information like the number of groups so that they can obtain the causal structures for known groups (or environments) as the shared causal structures for different clusters \cite{hu2018causal}. Then the learned shared structures are used for new samples clustering, so they cannot handle outlier samples. 

\begin{figure}[t]
	\centering
	\includegraphics[width=0.9\linewidth]{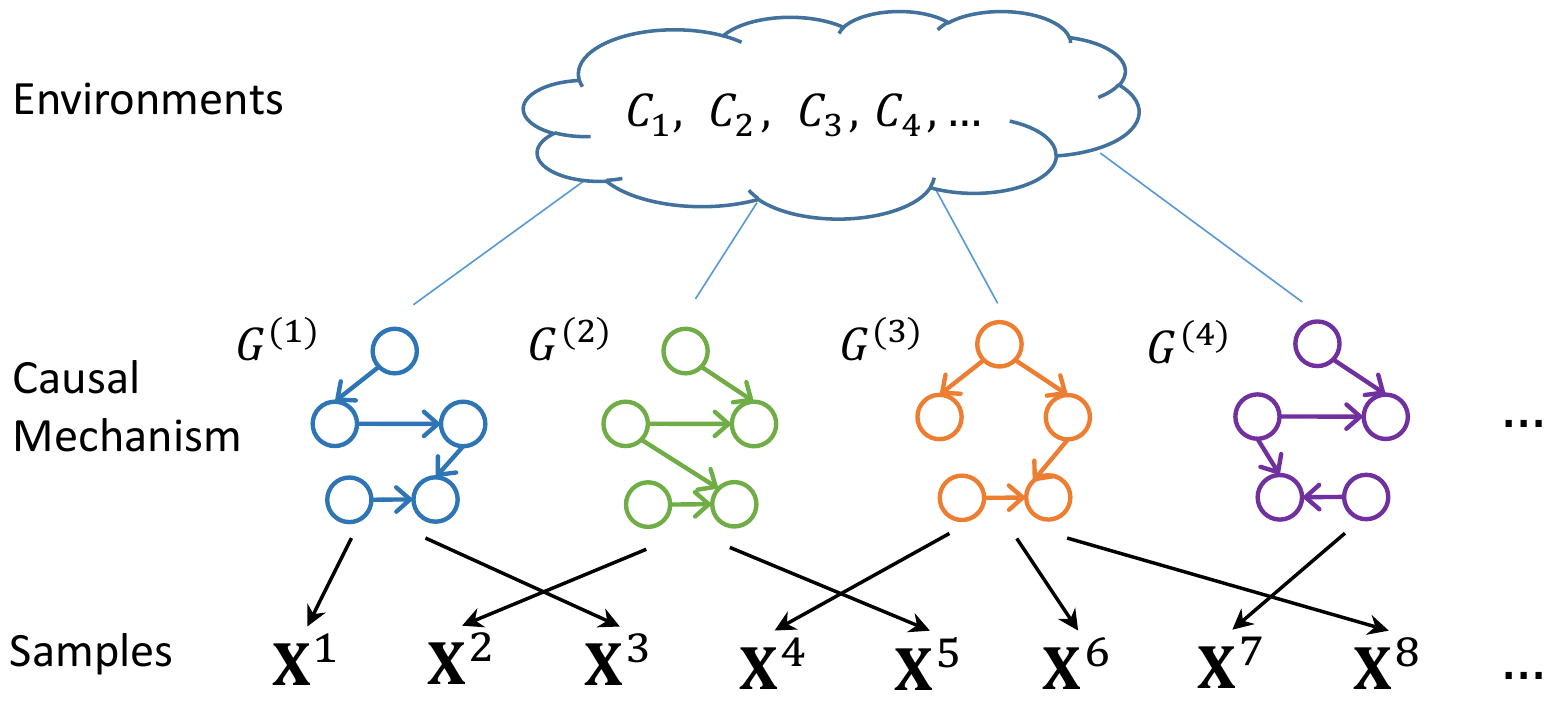}
	\caption{An illustration of the data generation process from multiple unknown environments. $C_1,C_2,\dots$ denote the multiple environments, $G^{(1)},G^{(2)},\dots$ denote the corresponding causal structures for the unknown environments, and $\mathbf{X}^s$ ($s\in \{1,2,\dots,\}$) denote sample for $s$-th subject. From the top to the bottom, multiple unknown environments imply multiple unknown causal mechanisms. Based on the causal mechanisms, multiple samples are generated. }
	\label{fig:gen_process}
\end{figure}

From the above analysis, we find that subjects grouping benefits from the causal mechanism, and the causal structure learning aims to recover causal mechanism from data. In another word, the causal mechanism acts as a mediator between the clustering and the structure learning. Therefore, we introduce a unified method, named Causal Clustering Structure Learning (CCSL) method, to leaning the causal structures from non i.i.d. data without requiring knowledge of the number of groups. This method simultaneously groups subjects that imply same causal structure into the same group and learns the causal structure of each group. In detail, inspired by the Chinese restaurant process, we provide a new causality-based grouping method named Causal Chinese Restaurant Process (Causal CRP). At the same time, the causal structural model can be estimated by a variational inference based approach. These two steps are embedded together and updated the parameters simultaneously, which outputs the groups information and the causal structure for each group.

\section{Related Work}
Conventional methods for causal structure learning from non i.i.d. data cluster samples into different groups by cluster methods, and then learn the causal structure from those i.i.d groups separately. So in this section, we review on samples clustering methods and causal structure learning methods. 

\subsection{Samples Clustering Methods}
Typical methods for clustering samples contains hierarchical clustering methods and partition clustering methods \cite{saxena2017review}. Partition clustering methods are relevant to our work, since they categorize different data from multiple subjects into $K$ groups based on distance measures like Euclidean distance and Dynamic Time Warping (DTW). Kmeans is a typical clustering approach, in which the given sample is clustered through a defined number of clusters $k$. Given $k$, every subject is assigned into the cluster to which the nearest cluster centroid belongs by Kmeans. Density-based spatial clustering of applications with noise (DBSCAN) \cite{Ester1996dbscan} and Ordering Points To Identify the Clustering Structure (OPTICS) \cite{ankerst1999optics} are density-based clustering method ordering based on the concept of maximal density-reachability. Besides, Chinese restaurant process (CRP) based methods \cite{aldous1985exchangeability} \cite{blei2011distance} provide Bayesian nonparametric clustering algorithms for high dimensional data analysis, and there are many extension methods based on CRP. Existing clustering methods consider similarity of variables as measures, which may be unstable when the samples collected from different environment. As a result, this method cluster non-i.i.d. samples together with several spurious similarities between different samples. Instead, causal structure among variables is more stable for each subject, because it contains the data generation process. So some studies \cite{hu2018causal}\cite{huang2019specific} used the causal edges of structure as features for clustering, but they need to recover the causal relationships for every subject first. Additionally, they need to know the number of groups, thus their performance is affected by the outlier samples. 

\subsection{Causal Structure Learning Methods}
The methods of learning causal structure include constraint-based methods and function-based methods. Constraint-based methods \cite{spirtes2000causation} use (conditional) independence test to remove independent causal relationships, and then orient the causal direction by using V-structure and Meek rules. But the main disadvantage of this method is the remaining Markov Equivalent class, i.e., some casual directions imply the same (conditional) independence conditions so that they can not be determined. Function-based methods \cite{shimizu2006linear} focus on the data generating process, then propose a structure causal model with some assumptions and provide a method to estimate the causal model. Considering the temporal information, the causal structure learning methods based on time series data are generated, such as PCMCI \cite{runge2018causal} \cite{runge2019detecting}, VAR-LiNGAM \cite{hyvarinen2010estimation}, dynotears \cite{pamfil2020dynotears} and so on. Nevertheless, these methods assume that the data were collected from an independent and identical distribution (i.i.d.), which fail to recover the different causal structures for different subjects.

Some researchers care about the above problem and present methods for relaxing the i.i.d. assumption. One kind of method is based on the mixture of DAGs \cite{pashami2018causal}\cite{saeed2020causal}\cite{zhang2020unmixing}. The idea of \cite{pashami2018causal}\cite{saeed2020causal} introduce an extra variable or using graphical modeling to reconstruct the graph. Because these methods are related to the constraint-based method, their results naturally remain the Markov equivalent class problem. For the linear, Gaussian data, Zhang et al. \cite{zhang2020unmixing} proposed a method to discover the causal relationship from data that are collected from mixed distributions. Another kind of method relies on cluster causal functional models, such as Specific and Shared Causal Model (SSCM) \cite{huang2019specific}, ANM Mixture Model (ANM-MM) \cite{hu2018causal}. The common idea of these methods is estimating the causal graphs for different subjects, and then clustering the estimated graphs by clustering methods. But these methods cannot fully utilize the data in the same group and require grouping information of samples, which are unreliable in the case there is a new outlier when clustering. In contrast to the existing method, our approach considers causal information during clustering and structure learning simultaneously, thereby solving the above issues.

\section{Causal Structure Learning from multiple unknown environments data}

In this section, we propose the causal clustering structure model, and then provide a practical solution to estimate this model. At last, the theory analyses on identifiablity and soundness are also provided.


Given observed time series (data) $\mathbf{X}=\{\mathbf{X}^1,\mathbf{X}^2,\dots,\mathbf{X}^n\}$ with time length $T$, they can be clustered into $q$ groups and imply different causal structures in different groups. Among these data, the number of variables in each $\mathbf{X}^{s} \in \mathbf{X}$ is $m$. Supposed that we do not know which observed time series (data) belongs to which group. In this paper, we aim to group $\mathbf{X}$ into $q$ clusters (Note that $q$ is not fixed but determined automatically) with the same causal structure, and in the meanwhile, learn the same causal structure over $m$ variables within the same groups.

To solve the above problem, we propose a unified model, named causal cluster structural model, which groups all subjects and learns the causal structures simultaneously. For the former, we introduce a new causal cluster discovery method named Causality-related Chinese Restaurant Process. For the latter, we learn the causal structure for each group based on the Structural Equation Model. An illustration of the proposed CCSL model is given in Figure \ref{fig:causal_CRP}. The detail of these methods will be introduced in the following subsections.

\subsection{Cluster Generation Model}
According to the problem that data for the subjects in the same group implies the same causal structure, we use a cluster to represent a causal structure. Suppose the data for several subject $\mathbf{X}=\{\mathbf{X}^1,\mathbf{X}^2,\dots,\mathbf{X}^s,\dots \}$ implies same causal structure $\mathbf{G}^k$. Let $c_s$ denote the index of the cluster to which they belong. Then, the causal clustering discovery task depends on the calculation of the conditional probability $p(c_s = k \mid \mathbf{X}^s)$. Note that we do not know the causal structures for each cluster. Considering the conditional probability, it can be incorporated into the CRP-based method that is based on the Bayesian nonparametric algorithm. Inspired by the CRP algorithm, we propose a Causality-related Chinese Restaurant Process (CausalCRP) for causal cluster discovery.

Before introducing CausalCRP, we provide a brief introduction of the traditional Chinese Restaurant Process (CRP). CRP is a discrete-time stochastic process, in which the probability of a customer sitting at a table is computed from the number of other customers already sitting at that table \cite{aldous1985exchangeability}. Considering $N$ customers, let $c_s$ denote the table assignment of the $s$-th customer and $n_k$ denote the number of customers sitting at table $k$. Assume that the customers $\{c_1,c_2,\dots, c_{(s-1)} \}$ occupy $K$ tables, the CRP draws the probability of each $c_s$ assigned to table $k$ as
    
\begin{equation}
    p(c_s=k|c_{1:(s-1)},\alpha) \propto 
        \begin{cases}
           n_k & \text{for $k \leq K$},  \\
           \alpha &  \text{for $k=K+1$}, 
        \end{cases} 
\label{CRP}
\end{equation}
where $\alpha$ is a given scaling parameter. When all customers have been seated, their table assignments provide a random partition.

Inspired by the CRP, we propose the Causality related Chinese Restaurant Process (Causal CRP) that considers the causal relationships among variables within groups. In Causal CRP, a table implies a causal structure. So the subjects that are assigned to the same table imply the same causal structure.
Let $c_s$ denote the group assignment of the $s$-th subject, then the probability of each $c_s$ assigned to table $k$, $p(c_s=k|c_{1:(s-1)},\alpha)$, is formalized as

\begin{equation}
    p(c_s=k|c_{1:(s-1)},\alpha) \propto 
        \begin{cases}
           P(c_s = k \mid \mathbf{X}^{s}) & \text{for $k \leq K$},  \\
           \alpha &  \text{for $k=K+1$}. 
        \end{cases} 
\label{eq:causalCRP}
\end{equation}
where $\alpha$ is a given scaling parameter, and $P(c_s = k \mid \mathbf{X}^{s})$ implies causal structure information.


Based on the given data, we can cluster samples into $q$ (determinated by the Causal CRP) groups, by estimating $P(c_s = k \mid \mathbf{X}^{s})$, for $k=1,2,\dots$, where
\begin{equation}
    P(c_s = k \mid \mathbf{X}^{s}) \propto P(\mathbf{X}^{s} \mid c_s = k) P(c_s = k).
    \label{eq:c_s=k}
\end{equation}


\begin{figure}[t]
	\centering
	\includegraphics[width=0.9\linewidth]{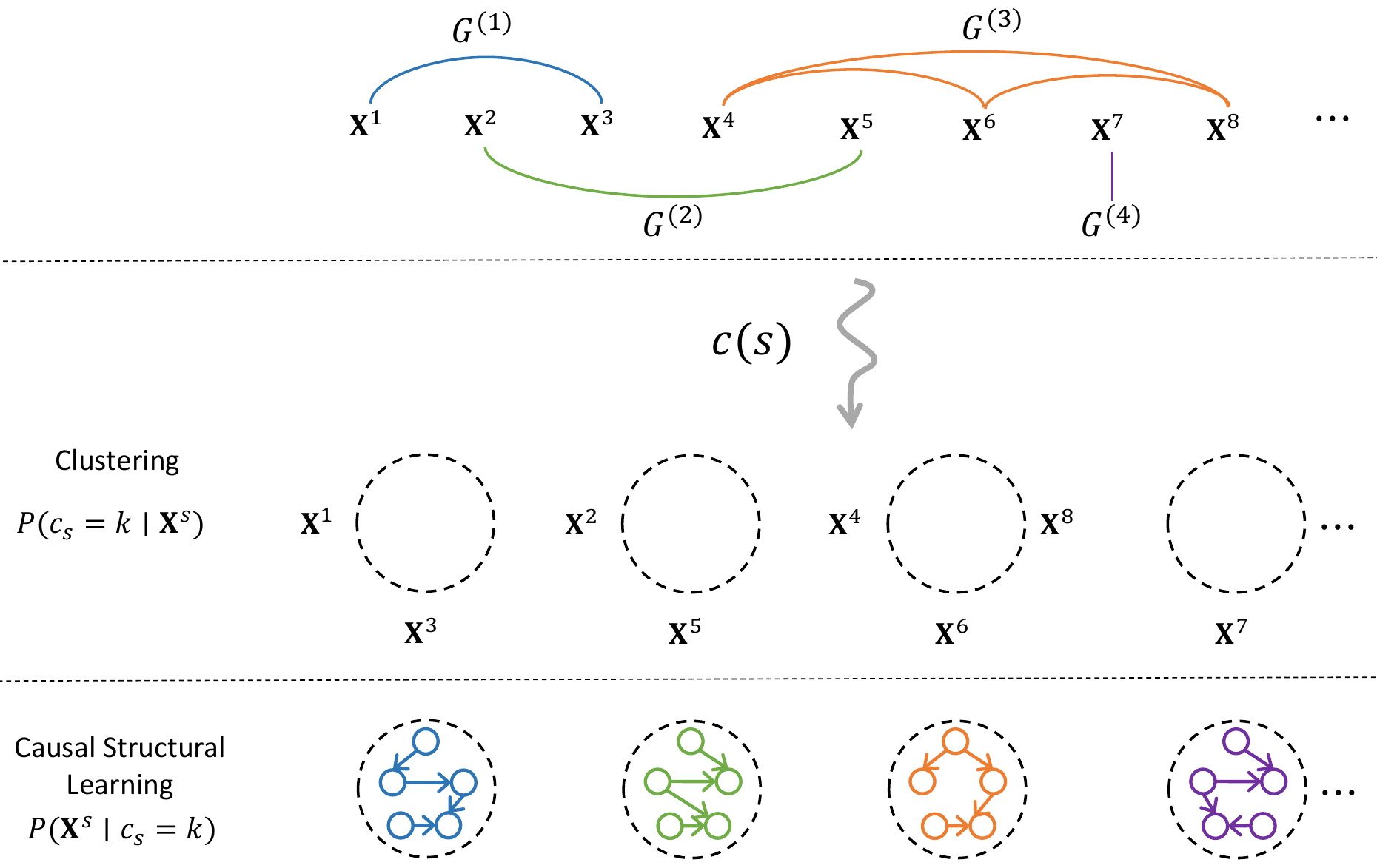}
	\caption{An illustration of the Causal Clustering Structure Model. The process operates at the level of subject assignments, where each subject is assigned to the table according to Equation (\ref{eq:causalCRP}). The blue lines connect the subjects that have the same or most similar causal structures. $G_i$ denotes the causal structure that is implied by the subjects in the $i-$th table.}
	\label{fig:causal_CRP}
\end{figure}

\subsection{Causal Structure Model for Each Cluster}

When using causal CRP, $P(c_s=k \mid \mathbf{X}^s)$ is calculated based on the known causal structure for table $k$. The causal structure is not given and should be learned from data. Considering the data generation mechanism, an intuitive way is to utilize the Functional Causal Model (FCM) to construct the causal relationships among observed variables. But there is a problem that the data is non-iid. So we need to propose a new specific functional causal model for each subject or cluster.  

In this paper, we focus on the the causal relationships among variables are linear and the noise of each variable is non-Gaussian, which are the common data generation assumptions for FCM. Considering the time series data and their time lag influences and instantanous effect, we propose the functional causal model with instantaneous effect and time-lagged effect for each individual subject as \cite{hyvarinen2010estimation}. Let $x_i^s (t)$ denote $i$-th variable of subject $s$ at time $t$ in a group. Assume the time lag of causal effect is $p_l$, then the FCM with instantaneous and time-lagged effect is defined as

\begin{equation}
    x_i^s (t) =\sum_{j\in pa_i^s} b_{ij}^s x_j^s(t)+\sum_{p=1}^{p_l}\sum_{j\in L_{i}^{s}} a_{ij,p}^{s} x_{j}^s(t-p)+e_{i}^{s}(t),
    \label{eq:model_i}
\end{equation}
for $i=1,2,\dots,m$, where $b_{ij}^s$ denotes the instantaneous causal influences from variable $j$ to $i$ in the $s$-th subject, $pa_i^s$ denotes the set of indexes parents of $x_i^s$, $a_{ij,p}^{s}$ represents the causal influences from variable $j$ to $i$ with time lag $p$, and $L_{i}^{s}$ denotes the set of indexes of parents of $x_i^s$ with time lag. Each subject has fixed causal coefficients $b_{ij}^{s}$ and $a_{ij,p}^{s}$, which is almost same in the same group and may be different in the different groups. The noise term $e_{i}^{s}$ is non-Gaussian, and independent with each other. In the same time, $e_{i}^{s}$ is independent of $x_j^s(t)$ and $x_j^s(t-p)$, for all $j,p\in \mathcal{N}^+$. 

In the matrix form, Eq. (\ref{eq:model_i}) can be written as
\begin{equation}
    X^{s}(t) = B^{s}X^{s}(t)+\sum_{p=1}^{p_l}A_{p}^{s}X^{s}(t-p)+E^{s}(t),
\end{equation}
where $B^{s}$ is the $m\times m$ causal strength matrix of instantaneous effect, $A_{p}^{s}$ is the $m\times m$ causal strength matrix of time-lagged effect, and $E^{s}(t)$ is the noise term that is independent with that of other variables.

We assume that the causal structures for different subjects in the same group are the same or mostly the same. Let $X^{(k)}(t)$ denote the variables in group $k$ at time $t$. Thus, for each group $k$, the group-specific causal model can be formalized as
\begin{equation}
    X^{(k)}(t) = B^{(k)}X^{(k)}(t)+\sum_{p=1}^{p_l}A_{p}^{(k)}X^{(k)}(t-p)+E^{(k)}(t),
    \label{eq:model_group}
\end{equation}
where $B^{(k)}$ is the causal strength matrix of instantaneous effect, $A_{p}^{(k)}$ is the time-lagged effect at time $(t-p)$, and $E^{(k)}(t)$ is the noise term for each group $k$ at time $t$.

\section{Practical Solution}

Based on the above analysis, we propose a novel method for causal clustering structure learning (CCSL) by using variational inference. First, we formalize the objective function and provide the prior on parameters. Then, the detail of the algorithm is provided.

\subsection{Objective Function and Priors}
Suppose that the observed multiple groups time series data $\mathbf{X}$ are generated by different causal mechanisms. Let $\Theta$ denote a set of all model parameters, i.e., $\Theta = \left\{\pi_{k}, \, \mu_{k, i j}, \, \sigma_{k, i j},\, \nu_{k, i j, p},\, \omega_{k, i j, p}, \, \pi_{ k^{\prime}}^{E^{(k)}}, \, \vec{\mu}_{ k^{\prime}}^{E^{(k)}}, \, \Sigma_{ k^{\prime}}^{E^{(k)}} \right\}$, where $k=1,2,\dots, q$, $i,j = 1,2, \dots, m$ and $p=1,2,\dots, p_l$. 
Given observed data, the marginal log-likelihood or the evidence as
\begin{equation}
    \begin{aligned}
    \mathcal{L}(\mathbf{X}) & = \log \int\int p(X,B,\{A_p\})dBd\{A_p\}\\
    & = \log \int\int p(B,\{A_p\} \mid X) p(X) dBd\{A_p\},    
    \end{aligned}
    \label{eq:L_x}
\end{equation}
where $\{A_p\} = \left \{ \{a_{ij,p} \}_{p=1}^{p_l}, i,j\in \{1,2,\dots,m \} \right \}$ and $B = \{b_{ij} \}_{i,j=1}^m$.

The variable inference (VI) method approximates the posterior distribution of the model parameters given the observations. So we use VI method to solve the above objective function. To obtain a tractable posterior distribution $p(B, \{A_p\} \mid X)$, we make the mean field assumption and approximate $p(B, \{A_p\} \mid X)$ with the factorized distribution $q_{\Theta}(B, \{A_p\}) =  q_{\Theta}(B)q_{\Theta}(\{A_p\})$. Using the factorized posterior distribution, we can obtain the posterior of instantaneous causal strength $B$ and time-lag causal influence $\{A_p\}$ independently.

Then, the log likelihood of observed data can be written as:
\begin{equation}
    \log p(X) = \mathcal{L}(q_{\Theta}) + KL(q_{\Theta} \| p),
    \label{eq:log_px}
\end{equation}
where $\mathcal{L}(q_{\Theta})$ is the evidence lower bound (ELBO), which can be formalized as
\begin{equation}
    \begin{aligned}
    \mathcal{L}(q_{\Theta})  = & \int \int q_{\Theta}(B, \{A_p\}) \log \left \{ \frac{p(X, B, \{A_p\})}{q_{\Theta}(B, \{A_p\})} \right \}d B d {\{A_p\}} \\
     = &  E_{q_{\Theta}(B,\{A_p\})}[\log p(X,B, \{A_p\})- \log q_{\Theta}(B, \{A_p\})] \\
     = &  E_{q_{\Theta}(B,\{A_p\})} \log p(X \mid B, \{A_p\}) ]\\
     + & E_{q_{\Theta}(B,\{A_p\})}[ \log p(B,\{A_p\}) - \log q_{\Theta}(B, \{A_p\}) ] \\
     = &  \mathcal{L}_{ell} + \mathcal{L}_{kl},
    \end{aligned}
    \label{eq:L_q}
\end{equation}
where $\mathcal{L}_{ell} = E_{q_{\Theta}(B,\{A_p\})}[\log p(X \mid B, \{A_p\})]$ and $\mathcal{L}_{kl} = KL(q_{\Theta}(B,\{A_p\} \| p(B,\{A_p\})$.

In order to make the $q_{\Theta}(B,\{A_p\})$ close to $p(B,\{A_p\} \mid X)$, we would minimize the KL divergence $KL(q_{\Theta} \| p)$. But this KL divergence contains the posterior density that we want to infer, so it is intractable. From Eq. (\ref{eq:log_px}), we can instead of maximize the ELBO $\mathcal{L}(q_{\Theta})$ (i.e., Eq. (\ref{eq:L_q})). 

Let $T$ denote the sample size for $\mathbf{X}$. For a certain group, the marginal probability of $p(X \mid B,\{A_p\})$ is
\begin{equation}
    \begin{aligned}
    & p(X \mid B, \{A_p\}) \\
     = &  p(X(1:T) \mid B, \{A_p\}) \\
     = & \left[ \prod_{t=p_l+1}^{T}p(X(t) \mid X(t-p_l:t-1), B, \{A_p\}) \right ] \\
     & \cdot \left[ \prod_{t=2}^{p_l} p(X(t)\mid X(1:t-1), B, \{A_p\}) \right ] \cdot p(X(1)). \\
    \end{aligned}
    \label{eq:p_x_marginal}
\end{equation}

To solve the objective function, we provide the prior on parameters. Considering the variation across groups and a similar causal model within each group, we regard the causal strength $b_{ij}^{(k)}$ in Eq. (\ref{eq:model_i}) as a random variable, and assume that $b_{ij}^{(k)}$ follows a Gaussian distribution in each group, while the Gaussian distributions vary across groups.
Then, the probability distribution of $b_{ij}^{(k)}$ can be formalized as
\begin{equation}
    \begin{aligned}
    P(b_{ij}^{(k)}) =  \mathcal{N}(b_{ij}^{(k)}\mid \mu_{k,ij},\sigma_{k,ij}^2),
    \end{aligned}
    \label{eq:prior_b}
\end{equation}
where $\mathcal{N}(\cdot)$ denotes a Gaussian distribution with mean $\mu_{k,ij}$ and variance $\sigma_{k,ij}^2$. Similarly, for all $i$, $j$, $p$ of $a_{ij,p}^{(k)}$ in group $k$, we have
\begin{equation}
    \begin{aligned}
    P(a_{i j, p}^{(k)} ) = \mathcal{N}(a_{ij, p}^{(k)} \mid \nu_{k, ij, p}, \omega_{k, ij, p}^2).
    \end{aligned}
    \label{eq:prior_a}
\end{equation}

We also allow the distributions of noise terms vary across different groups but remain the same within the same group. More specifically, we model the non-Gaussian noise in each group with a mixture of Gaussian. Denote $C^{E^{(k)}}$ as the indicator of the Gaussian mixture of $E^{(k)}$, with $C^{E^{(k)}} = (c_1^{E^{(k)}}, c_2^{E^{(k)}}, \dots, c_{q'}^{E^{(k)}})$, and thus in group $k$, the distribution of $E^{(k)}$ is
\begin{equation}
    \begin{aligned}
    P(E^{(k)} ) &=\sum_{k^{\prime}=1}^{q^{\prime}} P(c_s^{E^{(k)}}=k' ) P(E^{(k)} \mid c_s^{E^{(k)}}=k') \\
    & = \sum_{k^{\prime}=1}^{q^{\prime}} \pi_{k'}^{E^{(k)}} \mathcal{N}(E^{(k)} \mid \Vec{\mu}_{k'}^{E^{(k)}},\Sigma_{k'}^{E^{(k)}}),
    \end{aligned}
    \label{eq:prior_e}
\end{equation}
where $P\left(E^{(k)} \mid c_s^{E^{(k)}}=k' \right) = \mathcal{N}(E^{(k)} \mid \Vec{\mu}_{k'}^{E^{(k)}},\Sigma_{k'}^{E^{(k)}})$,  $ P(c_s^{E^{(k)}}=k' ) = \pi_{k'}^{E^{(k)}}$ and $\sum_{k'=1}^{q'}\pi_{k'} = 1$. 

With the prior of the parameters, the Eq. (\ref{eq:p_x_marginal}) can be written down as
\begin{equation}
\begin{aligned}
    & p(X(t) \mid X(t-p_l:t-1), B, \{A_p\}) =\\
      & |\!\det (I-B) |^{p_l} \!\cdot\! P_{E}[(I-B)X(t)\!-\!\sum_p \{A_p\} X(t-p_l\!:\!t-1)], \\
    & p(X(t) \mid X(1:t-1), B, \{A_p\}) = \\
     & |\det (I-B) |^{t} \cdot P_{E}[(I-B)X(t)\!-\!\sum_p \{A_p\} X(1\!:\!t-1)], \\
    & p(X(1)) = p(E),
    \label{eq:px_con}
    \end{aligned}
\end{equation}
where 
\begin{equation}
    \begin{aligned}
    & P_{E}[(I-B)X(t)-\sum_p \{A_p\} X(t-p_l:t-1)] = \\
    &\sum_{k'=1}^{q'} \! \pi_{k'}^{E} \mathcal{N}((I-B)X(t)\!-\!\sum_p \! \{A_p\} X(t-p_l \!:\! t-1) \!\mid \! \mu_{k'}^E,\Sigma_{k'}^E),    
    \end{aligned}
\end{equation}
and 
\begin{equation}
    \begin{aligned}
    & P_{E}[(I-B)X(t)-\sum_p \{A_p\} X(1:t-1)] =\\
    & \sum_{k'=1}^{q'} \pi_{k'}^{E} \mathcal{N}((I-B)X(t)-\sum_p \{A_p\} X(1:t-1) \mid \mu_{k'}^E,\Sigma_{k'}^E).
    \end{aligned}
\end{equation}

Similarly, for an individual subject, $P(X^s \mid c_s = k)$ is calculated as:
\begin{equation}
    \begin{aligned}
    & P(\mathbf{X}^{s} \mid c_s \!= \!k)= \\
    &\! \int\!\!\!\int\! P(\mathbf{X}^{s} \!\mid\! {A_p^{(k)}}\!,B^{(k)}\!, c_s \!=\! k)P({A_p^{(k)}}\!,B^{(k)} \!\!\mid\! c_s\! = \! k) d{A_p^{(k)}}d{B^{(k)}}.
    \end{aligned}
\end{equation}

The above integration does not have a closed form, and thus we use Monte Carlo integration. We sample $M$ values of ${A_p^{(k)}}$ and $B^{(k)}$ from $P({A_p^{(k)}},B^{(k)} \mid c_s = k)$. Then, according to Eq. (\ref{eq:p_x_marginal}), we can calculate $P(X^s \mid c_s = k)$ as follows.
\begin{small}
\begin{equation}
\begin{aligned}
    &P(\mathbf{X}^{s} \mid c_s = k) =\\
    &\frac{1}{M}\sum_{i=1}^{M}\mid \det (I-B^{(k,i)})\mid^{T_s}\\
   & \cdot \!\! \sum_{k'=1}^{q'} \!\!\pi_{k'}^{E} \!\!  \left \{\!\! \left[\prod_{t=p_l+1}^{T} \!\! \mathcal{N}(\! (I\!\!-\!\!B^{(k,i)}) \mathbf{X}_{t}^{s} \!\!-\!\!\sum_{p}\!\! A_p^{(k,i)} \!\! X_{t-p_l:t-1}^{s};\Vec{\mu}_{k'}^{E^{(k)}} \!\!,\!\Sigma_{k'}^{E^{(k)}} \!) \! \right]   \right.\\
  & \left. \cdot \left [ \prod_{t=2}^{p_l} \mathcal{N}((I-B^{(k,i)}) \mathbf{X}_{t}^{s} -\sum_{p}A_p^{(k,i)}X_{1:t-1}^{s};\Vec{\mu}_{k'}^{E^{(k)}},\Sigma_{k'}^{E^{(k)}}) \right] \right.\\
   & \left. \cdot  \mathcal{N}((I-B^{(k,i)}) \mathbf{X}_{1}^{s};\Vec{\mu}_{k'}^{E^{(k)}},\Sigma_{k'}^{E^{(k)}}) \right\},  
\end{aligned}
\label{eq:x_sMidc_s}
\end{equation}
\end{small}
where $A_p^{(k,i)}$ and $B^{(k,i)}$ denote the sampled $i$-th value from $P({A_p^{(k)}},B^{(k)} \mid c_s = k)$. Therefore,
\begin{equation}
    \begin{aligned}
    &P(c_s = k \mid \mathbf{X}^{s})\propto \\
    & \frac{\pi_{k}}{M} \! \sum_{k=1}^{M} \!\mid \! \det (I-B^{(k,i)}) \!\mid^{T_s} \!\cdot \! \sum_{k'=1}^{q'} \pi_{k'}\mathcal{N} \! \left((I-B^{(k,i)})) \mathbf{X}_{p+1:T_{s}}^{s} \right. \\
    &\left. -\sum_{p}A_p^{(k,i)}X_{1:T_s-p}^{s}\mid \Vec{\mu}_{k'}^{E^{(k)}},\Sigma_{k'}^{E^{(k)}} \right).
    \end{aligned}
\end{equation}

\subsection{Learning Algorithm}

Based on the objective function, our proposed method can be summarized as follows. Give the multiple group of observed data $\mathbf{X}$, CCSL method begins with a empty cluster (or table) with a scaling parameter $\alpha$. Then for the first considered sample $\mathbf{X}^i$, it is assigned into the first table with the probability $\alpha$, and the causal structure for that cluster is obtained. Based on this cluster, the cluster that following considered sample is assigned depend on the likelihood of the similarity on causal structures. If one sample is assigned into a existing cluster $j$, CCSL method updates the causal structure for cluster $j$. This process is repeated until convergence. 

All in all, during cluster samples by causal CRP, the causal structures for each group are estimated by using the variational inference method. The above two steps are nested together for clusters updating and causal structures updating until convergence. At last, the number of groups and causal structure for each group are determined. The pseudo-code for CCSL can be found in Algorithm \ref{alg:framework}.


\begin{algorithm}[ht]
	\caption{Causal Clustering Structure Learning (CCSL)}
	\label{alg:framework}
	\begin{algorithmic}[1]        
		\Require Observed data $\mathbf{X}$
		\Ensure Number of clusters $q$,  the cluster assignment of customers $\{c_s, s=1,2,\dots, n\}$, Causal structures $\{\mathcal{G}^j, j=1,2,\dots, q\}$ over observed variables       
		\State Initialize parameters $\Theta$ randomly;
		\State $q \gets n$;
		\For{all $\mathbf{X}^s \in \mathbf{X} $}
		\State  $k \gets s$;
		\State $c_s \gets k$;
		\State  append $\mathbf{X}^s$ into $\mathbf{X}^{(k)}$;
		\State   $\mid k\mid =1$; \Comment{ $\mid k \mid$ is the number of subjects in the $k$-th group} 
		\EndFor
		\Repeat
		\For{all $\mathbf{X}^s \in \mathbf{X} $}
		\For{all $i \in \{1,\dots,q\}$ }
		\State calculate $P(c_s = i \mid \mathbf{X}^s)$ according to Eq. (\ref{eq:causalCRP});
		\EndFor 
	\State $k \gets \mathop{\arg\max}\limits_{k} P(c_s = k \mid \mathbf{X}^s), k \in \{1,\dots,q\}$;
	\State $\mid k\mid \gets \mid k\mid -1 $
	\State $k \gets s$;
	\State append $\mathbf{X}^s$ into $\mathbf{X}^{(k)}$;
	\State $\mid k\mid \gets \mid k \mid +1 $
		
		\If{$\mid X^{(c_s)}\mid =0$}
		\State $q \gets q - 1$;
		\EndIf
		\State $\mathcal{G}^k \gets  \mathop{\arg\max}\limits_{ \Theta} \mathcal{L}(q_{\Theta})$ according to Eq. (\ref{eq:L_q}). \Comment{update $\Theta$ using gradient descent method with Adam optimizer\cite{kingma2015adam} }
		\EndFor
		\Until{convergence}
	\end{algorithmic}
\end{algorithm}

\section{Identifiablity}
In this section, first, we show that the functional causal model for each subject or group is also identifiable under mild assumptions. Second, we demonstrate that for clustering, the correct causal structure yields the highest likelihood.



\begin{theorem}
Given observation $\mathbf{X}^{(k)}$ for $k$-th group that is generated according to Eq. (\ref{eq:model_group}) where $k=1,2,\dots$, $\mathbf{B}^{(k)}$ and $\mathbf{A}_p^{(k)}$ are identifiable as the sample size $T \to \infty$.
\end{theorem}

\begin{proof}
We consider the following three cases where the identifiablity holds.
\begin{itemize}
    \item If the model does not contain the lagged effect, Eq. (\ref{eq:model_group}) actually is LiNGAM (Linear Non-Gaussian Acyclic Model), which is proved to be identified \cite{shimizu2006linear}. 

    \item If the model only contains the lagged effect, Eq. (\ref{eq:model_group}) is the generalization of Granger Causality \cite{granger1969investigating} in multiple variables.

    \item If the model contains both instantaneous effect and lagged effect, Eq. (\ref{eq:model_group}) becomes a structural vector autoregression (SVAR) model with non-Gaussian instantaneous effect \cite{hyvarinen2010estimation}.
    \begin{equation}
        \begin{aligned}
            X(t)= BX(t)+\sum_{p=1}^{p_l}A_pX(t-p)+E(t),
        \end{aligned}
    \end{equation}   
    which is also proved to be identified.
\end{itemize}

Thus, $\mathbf{B}^{(k)}$ and $\mathbf{A}_p^{(k)}$ in the causal model defined as Eq. (\ref{eq:model_group}) for each group is identifiable. 
\end{proof}

Note that when a group only has one subject, the data of the subject generated by Eq. (\ref{eq:model_i}) is the same as the data that is generated according to Eq. (\ref{eq:model_group}). So the causal structure for each subject is also identifiable. Based on the identified causal structures, the clustering results is also correct, which is guaranteed by the following theorem.
\begin{theorem}
Let sample size $ T_s \to \infty$. Given observed data $\mathbf{X}$ which is generated from $q$ different causal clustering structure model as Eq. (\ref{eq:model_i}), for each group data $\mathbf{X^s}$, $k$ is consistent with the ground-truth, when $P(c_s=k \mid \mathbf{X}^s)$ is asymptotically higher than the probability of subject $s$ assigned to the another cluster conditional on $\mathbf{X}^s$.
\label{the:cluster}
\end{theorem}

\begin{proof}
(Proof by Contradiction.) Suppose there was a group index $z$ satisfies: 1) $z\neq k$, 2)$P(c_s=z \mid \mathbf{X}^s)>P(c_s=k \mid \mathbf{X}^s)$. Then, according to Eq. (\ref{eq:c_s=k}), we have
\begin{equation}
    \begin{aligned}
    &  \frac{P(c_s=z \mid \mathbf{X}^s)}{P(c_s=k \mid \mathbf{X}^s)}\\
    = &   \frac{P(\mathbf{X}^s \mid c_s=z)P(c_s =z)}{P( \mathbf{X}^s\mid c_s=k)P(c_s=k) }\\
    = &   \frac{P(c_s=z)}{P(c_s=k)} \cdot  \frac{P(\mathbf{X}^s \mid c_s=z)}{P( \mathbf{X}^s\mid c_s=k)}.
    \end{aligned}
\end{equation}

Let $M$ denote the sample size, and $n$ denote the number of observed variables. Using the log function, we obtain
\begin{equation}
    \begin{aligned}
    & \log \frac{P(c_s=z \mid \mathbf{X}^s)}{P(c_s=k \mid \mathbf{X}^s)} =  \log  \frac{P(c_s=z)}{P(c_s=k)} + \log \frac{P(\mathbf{X}^s \mid c_s=z)}{P( \mathbf{X}^s\mid c_s=k)},
    \end{aligned}
\end{equation}
where
\begin{equation}
    \begin{aligned}
    & \log \frac{P(\mathbf{X}^s \mid c_s=z)}{P( \mathbf{X}^s\mid c_s=k)}\\
    = &  \sum_{i=1}^M  \sum_{j=1}^m  \log \frac{ P(x_{j}^s= \mathbf{X}_{i,j}^s \mid par(j,z))= \mathbf{X}_{i,par(j,z)}^s)}{ P(E_{i,j} = \mathbf{X}_{i,j}^s - f(\mathbf{X}_{i,par(j,k)}^s)}\\
    = &  \sum_{i}^M  \log \frac{ P(x^s= \mathbf{X}_{i}^s)}{ \prod_{j=1}^m P(E_{i,j} = \mathbf{X}_{i,j}^s - f(\mathbf{X}_{i,par(j,k)}^s)} \\
    = & \! -\!  M \! \cdot \! \mathbb{E}_{X^s} \log \frac{ P(X^s=\mathbf{X}^s)}{ \prod_{j=1}^m P(E_j = \mathbf{X}_{j}^s - f(\mathbf{X}_{j,par(j,k)}^s)}\\
    = & \! - \! M \!\! \cdot \! KL \!\! \left( \!  p(X^s \!=\!\mathbf{X}^s) \| \prod_{j=1}^m \! P(E_j \!= \!\mathbf{X}_{j}^s \! -\!  f(\mathbf{X}_{j,par(j,k)}^s) \!  \right)\!\!,
    \end{aligned} 
    \label{eq:compare}
\end{equation}
where $par(j,z)$ denotes the parent of $x_j$ with respect to $\mathcal{G}^{(z)}$.

According to the second condition $P(c_s=z \mid \mathbf{X}^s) > P(c_s=k \mid \mathbf{X}^s)$ and Eq. (\ref{eq:compare}), there should be $ \log \frac{P(c_s=z)}{P(c_s=k)} > KL(p(X^s) \| P(E_i = X_i^s - f(B^{(z)}, A_p^{(z)})))$. Then $\log \frac{P(c_s=z)}{P(c_s=k)} > 0$ due to the KL divergent is larger than 0. That is, $P(c_s=z)>P(c_s=k)$. 

But in fact, $X^s$ belongs to group $k$, i.e, $P(c_s=z)<P(c_s=k)$. So $P(c_s=z)>P(c_s=k)$ contradicts the supposition. 

Since the supposition is false, it states that the original statement is true.
\end{proof}

Note that Theorem \ref{the:cluster} is based on the assumption that the subjects belong to the same cluster imply the same causal mechanism. Under the identified causal structure estimation, we can always cluster the subject into the correct group, and obtain the correct causal structure.  

\section{Experiments}
To evaluate the correctness and effectiveness of our method, we conduct extension experiments on both synthetic data and real-world data. The source code of the proposed method is publicly available online \footnote{{https://github.com/DMIRLAB-Group/CCSL}}.

\subsection{Synthetic Data}
We randomly generated directed acyclic causal structures according to the Erdos-Renyi model \cite{erdds1959random} with parameter 0.3. To show the generality of the proposed method, we varied the number of variables for each causal graph with $\{6,8,10,12,14\}$, the sample sizes for each subject with $l_s = \{40, 60,80,100,120\}$, the number of groups with $q = \{2,3,4,5,6\}$ and the number of subjects with $n = \{20, 30,40,50,60 \}$. Motivated from the real-world scenario that brain connectives may be enhanced or inhibited in subjects with mental disorders, such as autism and schizophrenia, compared to typical controls, the parameters were set up as \cite{huang2019specific} in the following way:


$\mu_{k,ij} \sim \mathcal{U}(0.1, 0.4)$, $\nu_{k, i j, p}  \sim \mathcal{U}(0.1, 0.4)$, $\sigma_{k,i,j}^2 \sim \mathcal{U}(0.01,0.1)$, $\omega_{k, i j, p}^2 \sim \mathcal{U}(0.01,0.1)$, each $\vec{\mu}_{ k'}^{E^{(k)}} \sim \mathcal{U}(-0.6,-0.4) \bigcup \mathcal{U}(0.4,0.6)$, each $\Sigma_{ k'}^{E^{(k)}} \sim \mathcal{U}(0.2,0.5)$, $\pi_{ k'}^{E^{(k)}} \sim \mathcal{U}(0.3,0.6)$, and $\sum_{k=1}^{q} \pi_{k} =1$, $\sum_{k'=1}^{q'} \pi_{k'}^{E^{(k)}} =1$, where $\mathcal{U}(a,b)$ denotes a uniform distribution between $a$ and $b$. For each setting (a particular group size $q$, a particular sample size for each subject $T_s$, and the number of subject $n$), we generated 10 realizations.

For clustering, we evaluate the performance of the proposed CCSL, and compare with the clustering results of state-of-the-art clustering methods, which are as follows:
\begin{itemize}
    \item Kmeans \cite{macqueen1967some}: It is a well-known partition clustering algorithm, which assigns different subjects into $k$ clusters based on the criterion that each subject is closest to the center of its category. 
    \item DBSCAN \cite{Ester1996dbscan}: It is a density-based clustering method, assuming that the clustering category can be determined by the tightness of the sample distribution. The samples of the same category are closely connected between them, i.e., the samples of the same category are not far from any sample of this category.
    \item OPTICS \cite{ankerst1999optics}: It is a revision of DBSCAN, which addresses the disadvantage of DBSCAN on detecting meaningful clusters in data of varying density.
\end{itemize}

We use DTW and Euclidean distance as their measure criterion, so the baseline clustering methods are KMeans (DTW), KMeans (Euclidean), DBSCAN (DTW), and OPTICS (DTW). The brackets with each method indicate the chose distance measurement methods, e.g. KMeans (DTW) means the KMeans method with DTW measure. For the implementation of the baseline methods, we use the public codes from scikit-learn\footnote{\url{https://scikit-learn.org/stable/modules/clustering.html}}.

For causal discovery, we identify the causal structures of different groups estimated by our method. We compare it with the following baseline methods:
\begin{itemize}
    \item VAR-LiNGAM \cite{hyvarinen2010estimation}: It is a function-based method to estimate a structural vector autoregression model. Under the non-Gaussian assumption, it can identify the instantaneous and time-lagged linear effect between variables.
    \item PCMCI \cite{runge2018causal} \cite{runge2019detecting}: This method is a constraint-based method and assume no instantaneous effect between variables. It contains two stages. The first stage is based on the PC algorithm and recovers some potential causal relationships for each variable, and the second stage is using the Momentary Conditional Independence (MCI) test to remove some redundant edges.  
    \item DYNOTEARS \cite{pamfil2020dynotears}: It is a score-based method for learning dynamic Bayesian networks. With the acyclicity constraints and the SVAR model, it can learn instantaneous effect and time-lag influence between variables simultaneously. 
\end{itemize}

Though these methods assume the data are homogeneous and the causal model is fixed, we apply these methods on data of each subject separately to estimate the adjacent matrix, and then use the KMeans method to cluster the estimated adjacent matrices into different groups. 

In our method, we initialized the parameters randomly. In our experiments, the number of groups can be unknown. Let $\hat{\mathcal{G}}^k$ denote the estimated shared causal graph for the $k$-th group. It was determined as follows: $\hat{\mathcal{G}}^k_{i,j}=1$ if $|\hat{b}^k_{i,j}|>1$ or $|\hat{a}^k_{i,j,p}|>0.1$ or both, and $\hat{\mathcal{G}}^k_{i,j}=0$ if otherwise. Alternatively, one may use Wald test to examine significance of edges, as in \cite{shimizu2006linear}.

\textbf{Evaluation Metrices.} We evaluate the performance of our method in the terms of clustering and causal structure learning.
For clustering, we use the Adjusted Rand Index (ARI \cite{rand1971objective}) to measure the correction of the estimated groups. It measures the similarity between the estimated groups and the ground truth (the higher, the more accurate).
For causal structure learning, we use Area Under Curve (AUC) to measure the accuracy of learned causal graphs.



\begin{figure}[t]
\centering
\subfigure[Different number of variables]{
  \includegraphics[width=0.45\linewidth]{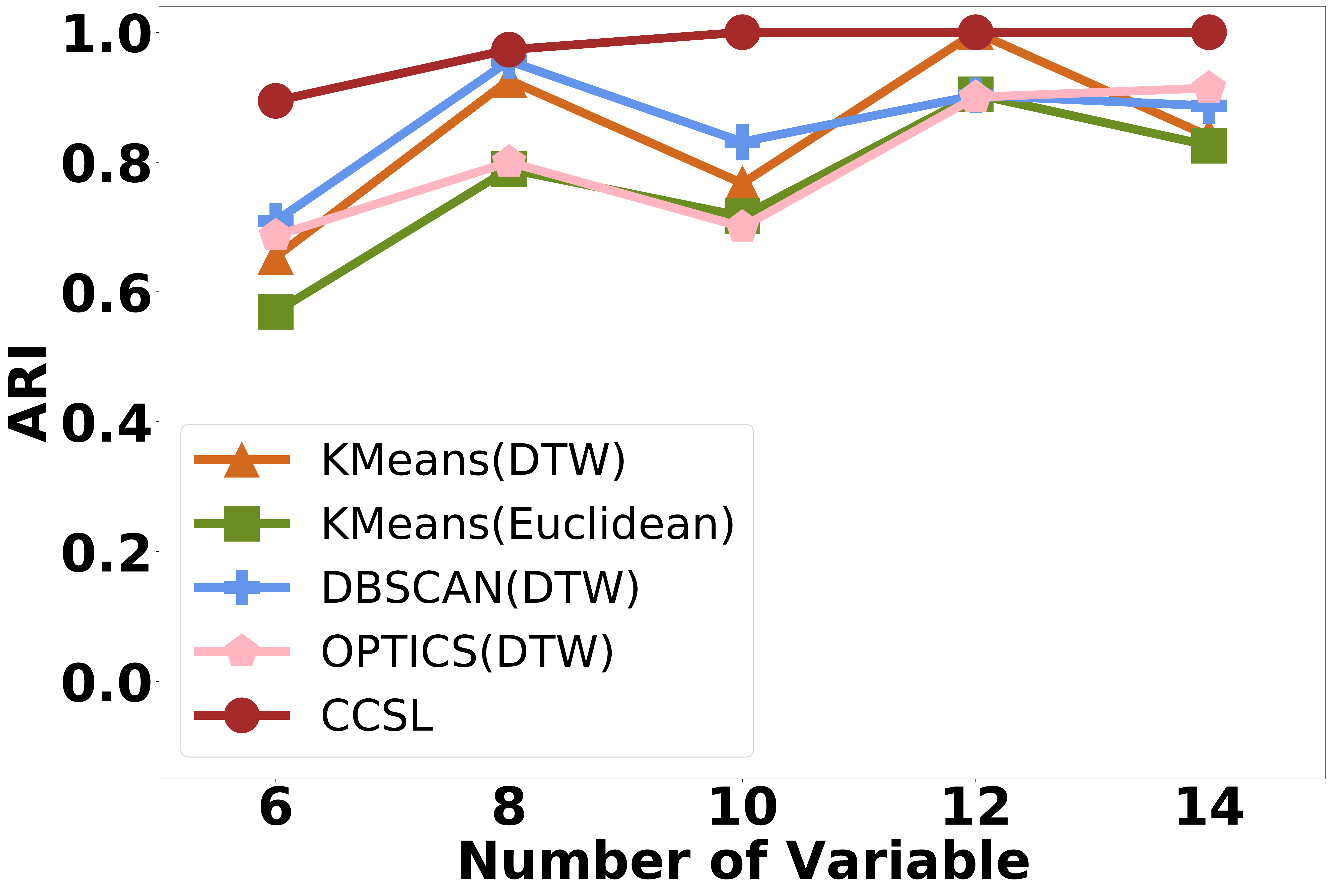}
  \label{fig:ARI_synthetic_variable}
}
\subfigure[Different sample sizes]{
    \includegraphics[width=0.45\linewidth]{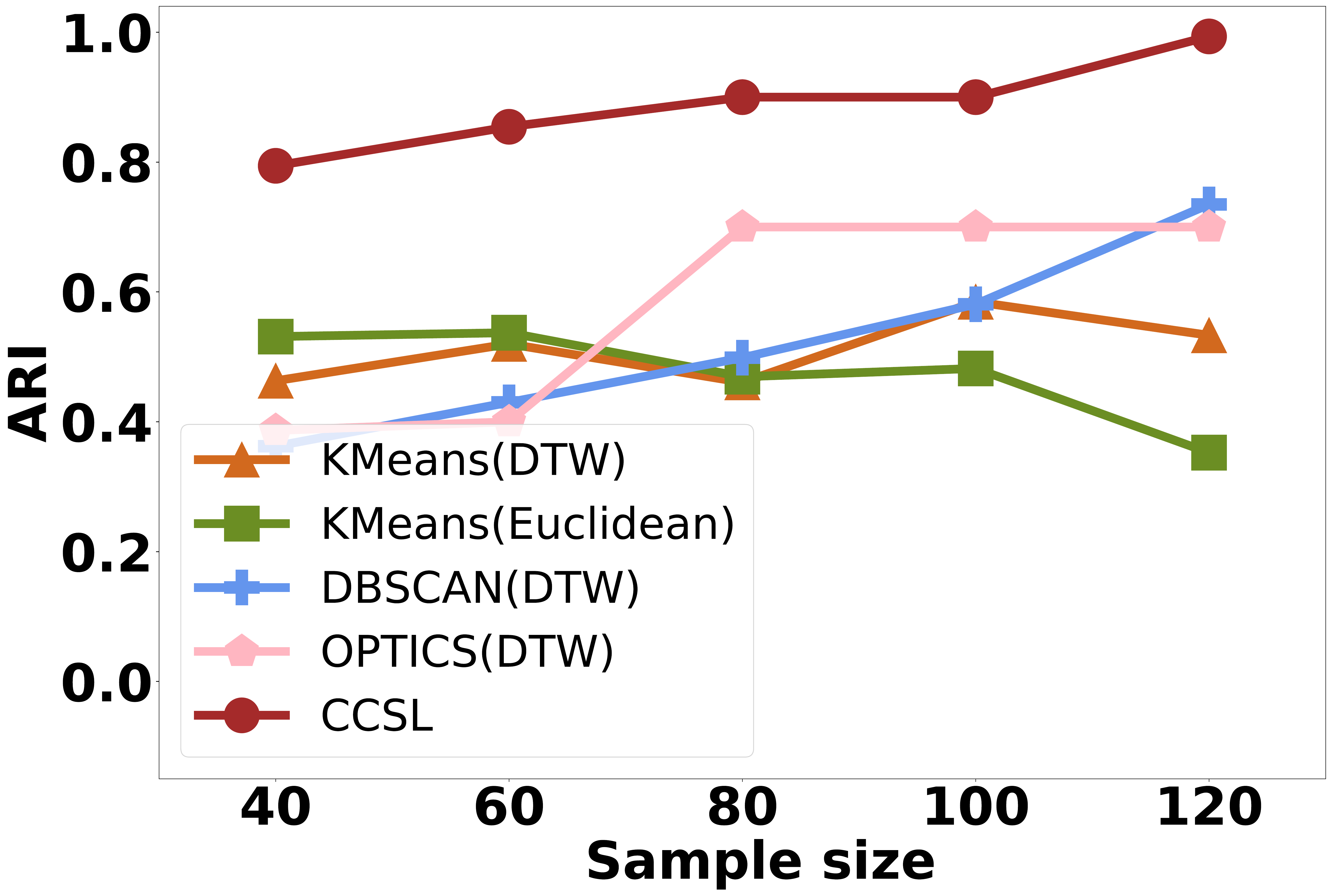}  
    \label{fig:ARI_synthetic_sample}
}
\subfigure[Different number of groups]{
    \includegraphics[width=0.45\linewidth]{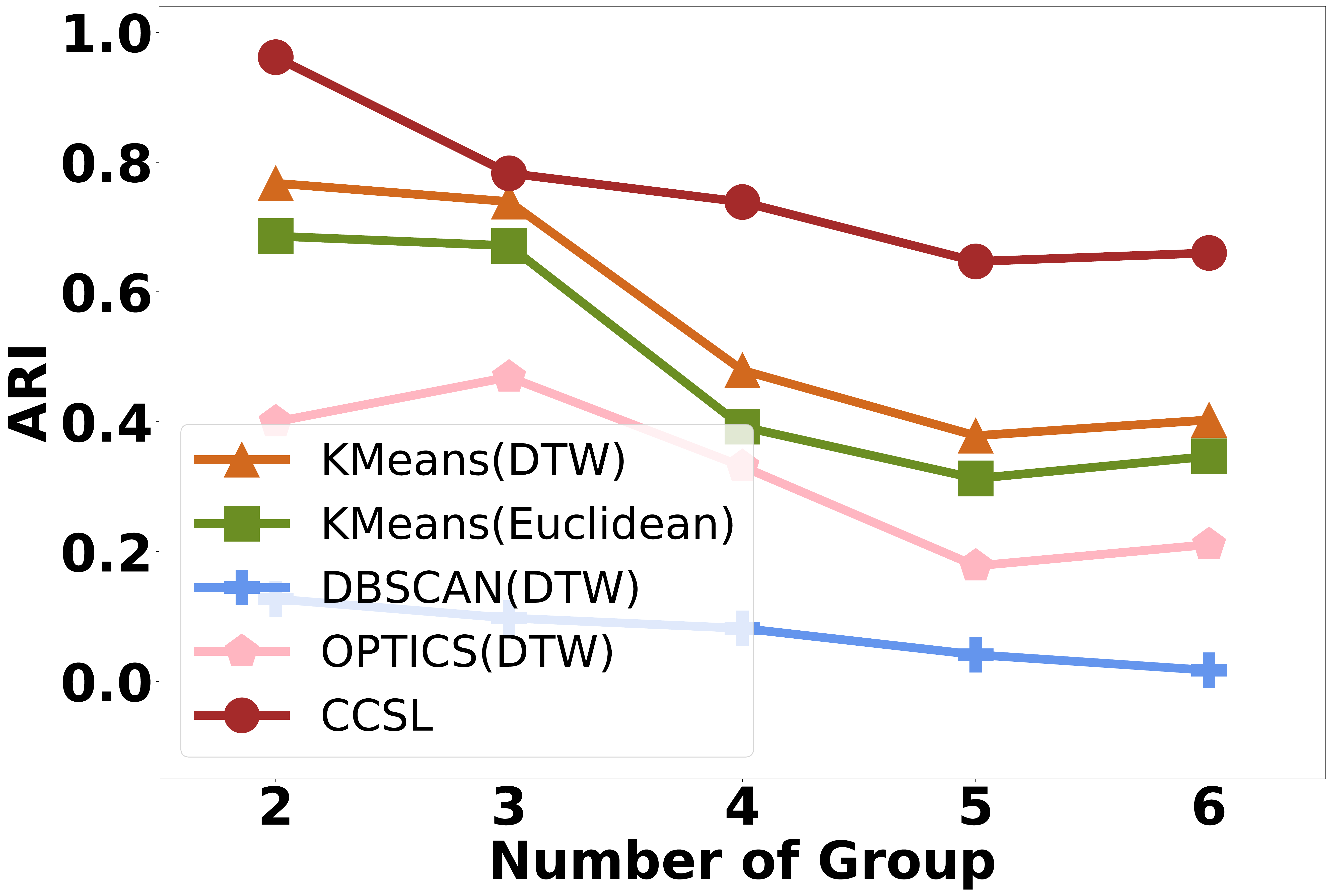}  
    \label{fig:ARI_synthetic_group}
}
\subfigure[Different number of subjects]{
  \includegraphics[width=0.45\linewidth]{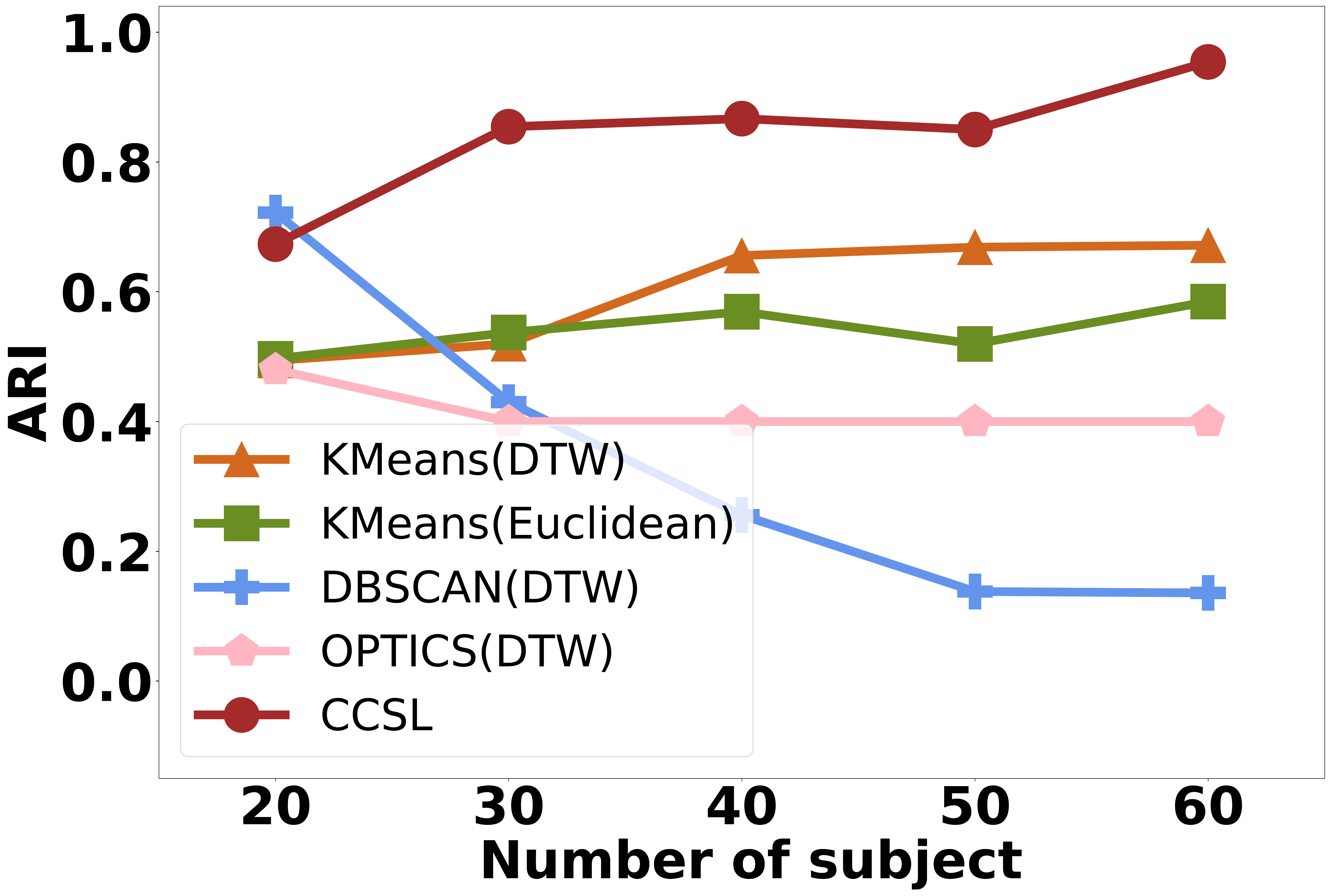}  
  \label{fig:ARI_synthetic_subject}
}
\caption{ARI under different settings compared with the baseline methods}
\label{fig.ARI_synthetic}
\end{figure}

\paragraph{\textbf{Sensitivity on clustering}}

The results of different clustering methods and our method are shown in Figure \ref{fig.ARI_synthetic}. Our method performs best in most cases, except where the number of subjects is $20$ in Figure \ref{fig.ARI_synthetic}(d). This is because our method considering the causal relationships among variables for different subjects, which is closer to the true mechanism. 
Figure \ref{fig.ARI_synthetic}(a) illustrates the ARI under the different number of variables when the number of subjects is $30$, the number of groups is $2$, the sample size is $60$. The ARIs of baselines are mostly between $0.6$ and $0.8$. From the results, one can see that our proposed method works well in the different number of variables. This verifies that more variables contain more causal information, which helps for clustering. The other methods change similarly due to using the same measure features for clustering. Figure \ref{fig.ARI_synthetic}(b) gives the results when sample size changes, CCSL obtains a much higher ARI (about 0.2) than others. The ARI of KMeans with DTW or Euclidean distance as a measured way decreases when the sample size increases. This reflects that some samples could affect the calculation of distance. Figure \ref{fig.ARI_synthetic}(c) shows that the ARI of all methods decreases When the number of groups increases. But the ARI of our method decreases much slower than that of others. All results show that our method performs well when there is a larger sample size, more subjects. Because this case can provide more information to learn the causal structures more correct, and the estimated structures help to cluster more correct.

\paragraph{\textbf{Sensitivity on causal structures learning}}

\begin{figure}[t]
\centering
\subfigure[Different number of variables]{
  \includegraphics[width=0.45\linewidth]{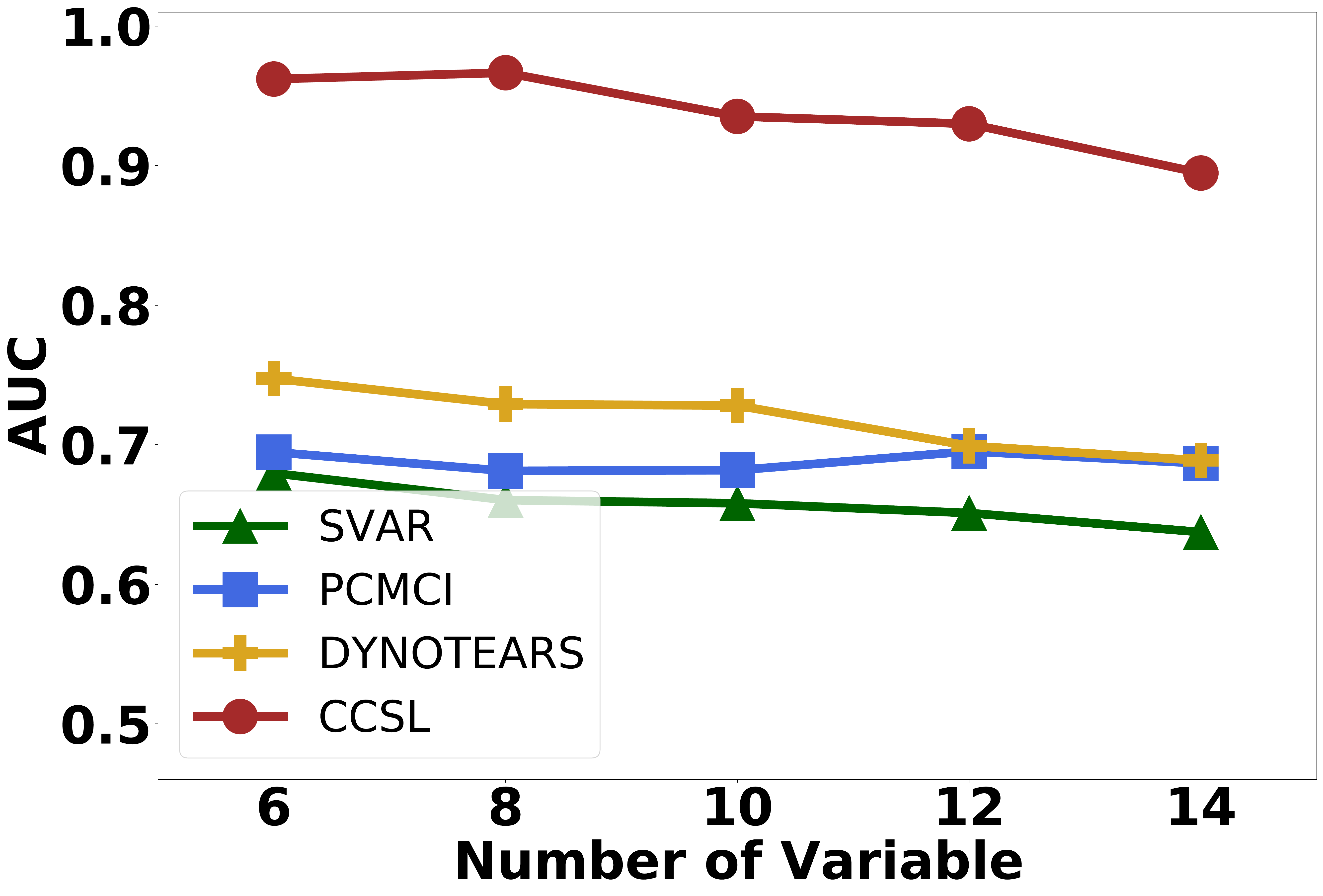}
  \label{fig:AUC_synthetic_variable}
}
\subfigure[Different sample sizes]{
  \includegraphics[width=0.45\linewidth]{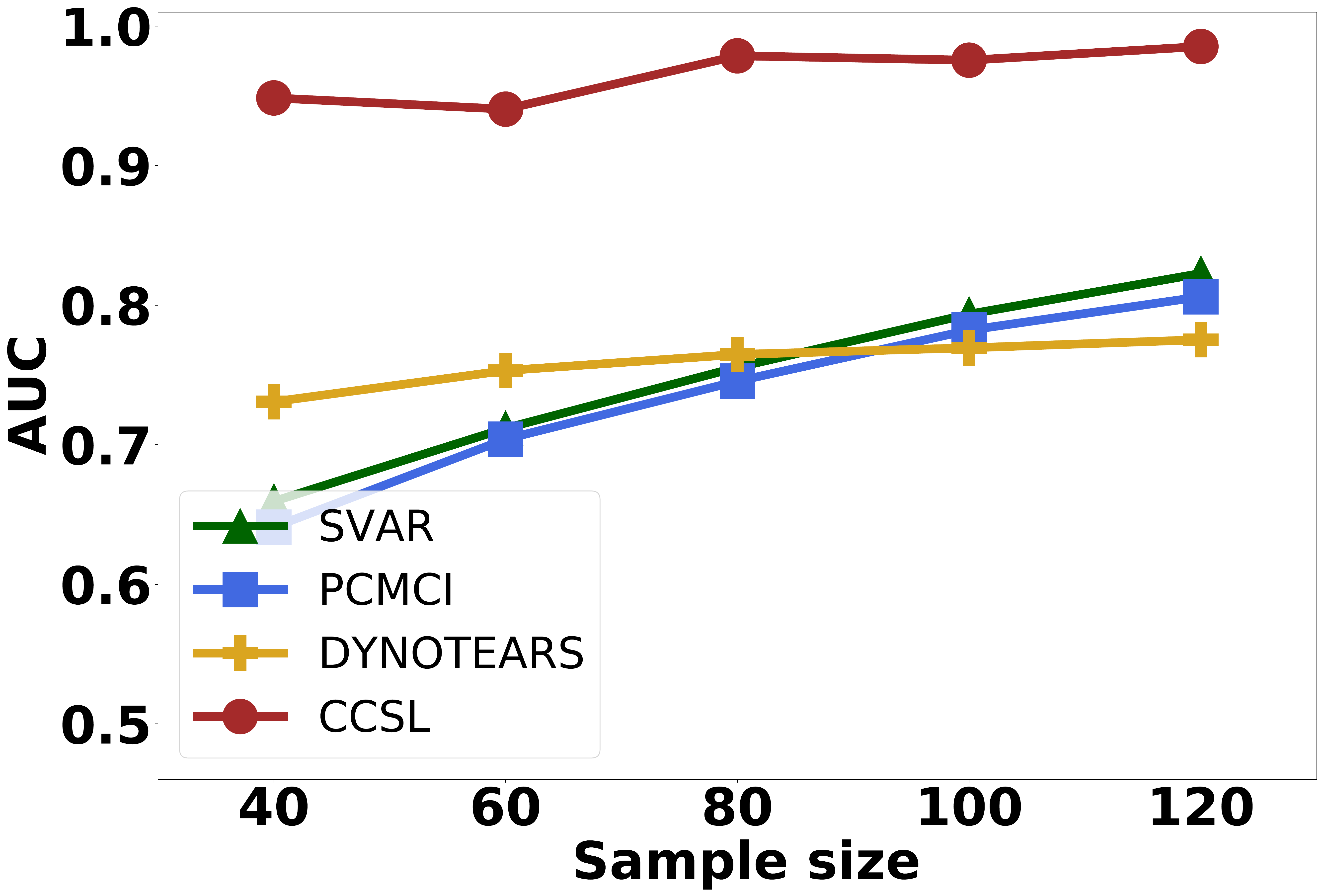}
  \label{fig:AUC_synthetic_sample}
}
\subfigure[Different number of groups]{
  \includegraphics[width=0.45\linewidth]{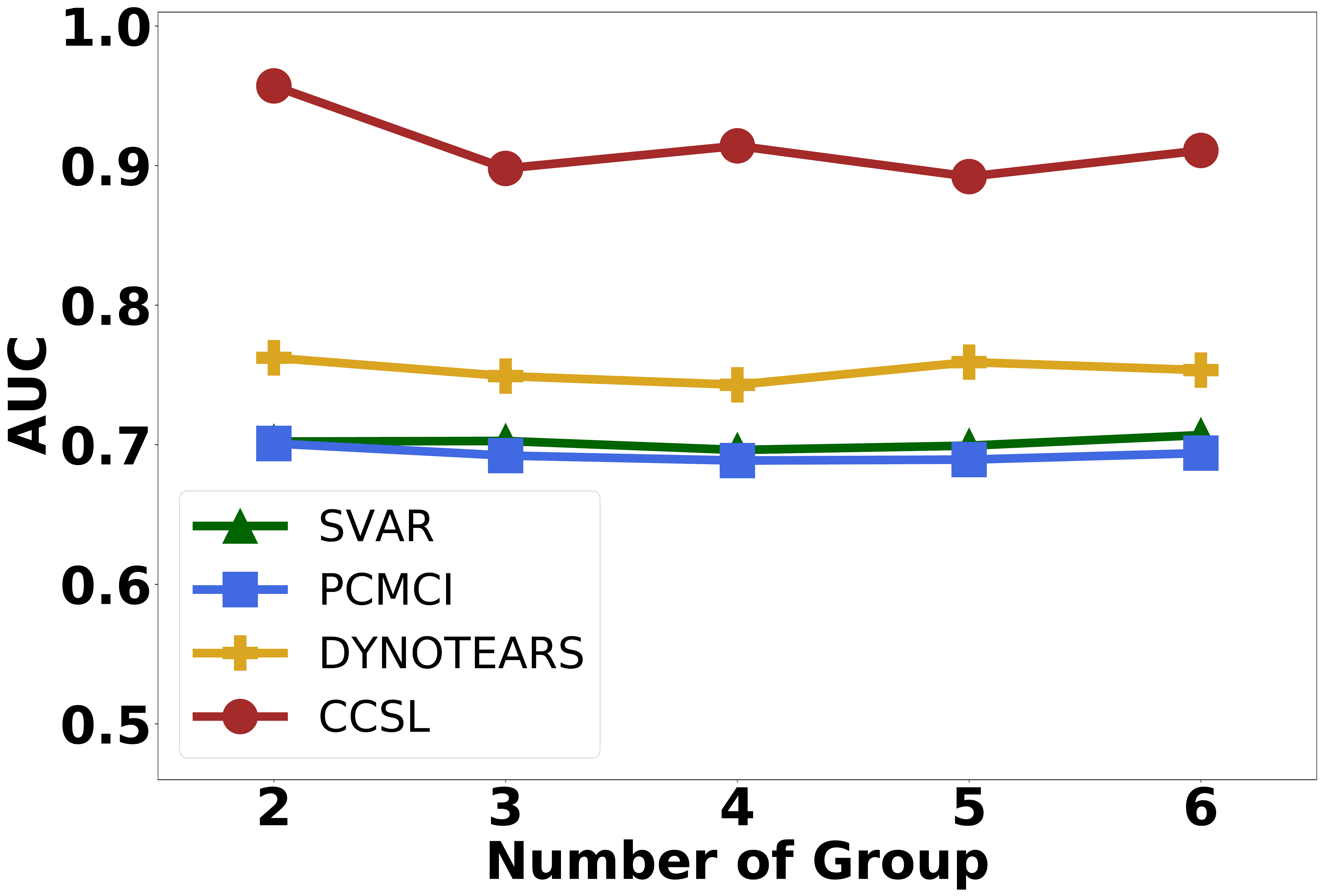} 
  \label{fig:AUC_synthetic_group}
}
\subfigure[Different number of subjects]{
  \includegraphics[width=0.45\linewidth]{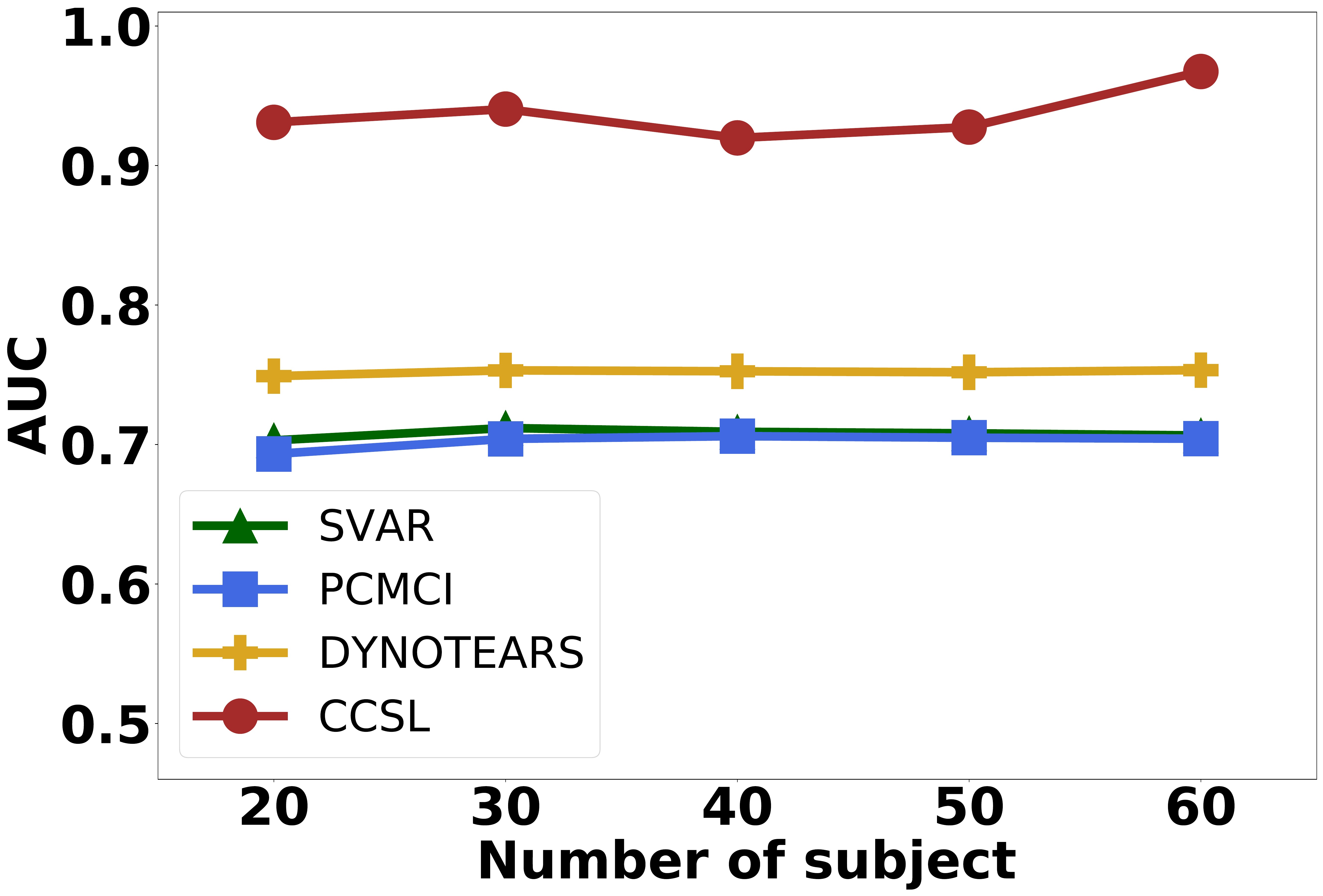} 
  \label{fig:AUC_synthetic_subject}
}
\caption{AUC under different settings compared with the Baseline Methods}
\label{fig.AUC_synthetic}
\end{figure}

Figure \ref{fig.AUC_synthetic} illustrates the results of causal structures learned by different methods. Our methods significantly outperform other baselines on all synthetic datasets. Dynotears performs better than other baselines in most cases, but the AUC of all baselines is less than 0.8. In detail, Figure \ref{fig.AUC_synthetic}(a) shows the performance of causal structure learning is slightly decrease when the number of variables. This is the same as the causal structure learning from i.i.d data. Figure \ref{fig.AUC_synthetic}(b) gives the results under different sample sizes. More sample guarantees the estimation reliable, expect for SVAR and PCMCI. For SVAR, the fitting ability highly depends on the sample. For PCMCI, the independence test needs large sample sizes to make sure the result is consistent with the theory. Figure \ref{fig.AUC_synthetic}(c) illustrates the results under the different number of groups. Although the AUC of all methods is stable in different cases, the results of our method are higher or equal to $0.9$ while the AUC of other baseline methods is lower than $0.8$. It verifies that the causal CRP does a great deal for causal structure learning. Figure \ref{fig.AUC_synthetic}(d) shows the AUC of our method increases as the number of subjects increases, while that of others is unchanged. Because the existence of more subjects can bring more samples, which can assist in the study of causal structure learning. All these results reflect that the datasets from the same group provide more samples to learn the causal structure, which proves the advantage and effectiveness of our method. Moreover, our method does not need to give the number of groups and can spontaneously learn the causal structures.

\paragraph{\textbf{Convergence analysis}}
We also analyze the convergence of our proposed algorithm when applying to the simulation data set. Due to the space limitation, we only show the log-likelihood obtained in every iteration when the number of subjects is $20$ and the sample sizes is $60$ as Figure \ref{fig:convergence}. As shown in Figure \ref{fig:convergence}, the log-likelihood obtained by our model converges at around 18 iterations. This result indicates that our method converges in few iterations and can easily be applied in real-world scenarios.

\begin{figure}
    \centering
    \includegraphics[width=0.98\linewidth]{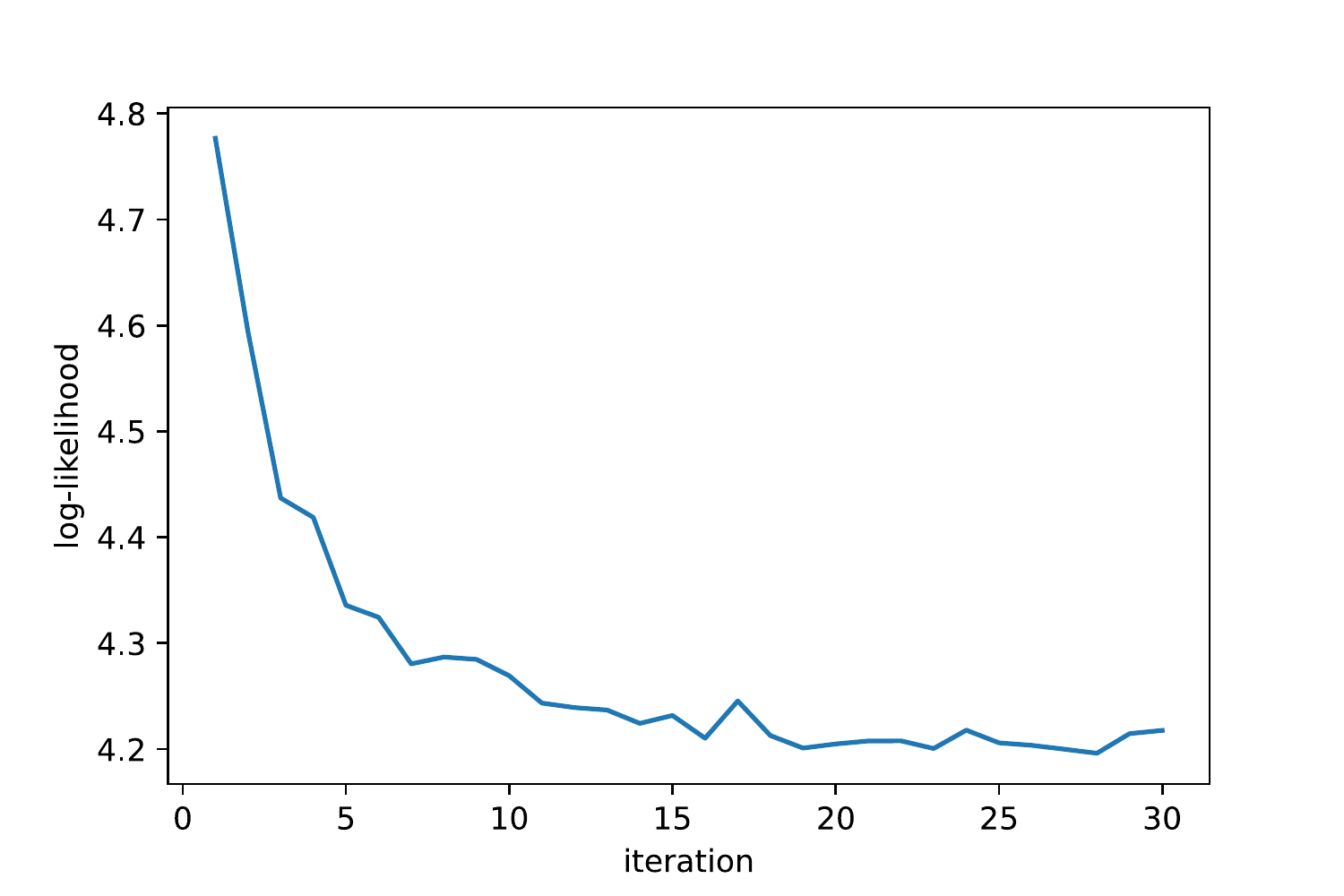} 
    \caption{Convergence analysis of our proposed model}
    \label{fig:convergence}
\end{figure}

\subsection{Real-world Data}
In this section, we apply our method to two real-world data to evaluate the performance of our method.

\paragraph{\textbf{fMRI data}} We apply our method to fMRI data\footnote{\url{https://openfmri.org/dataset/ds000031/}}. This data records the signal of six brain regions in the resting state: perirhinal cortex (PRC), parahippocampal cortex (PHC), entorhinal cortex (EC), subiculum (Sub), CA1, and CA3/Dentate Gyrus (CA3). It collects the data on $84$ successive days. According to the work in \cite{huang2019specific}, we treat the data of every day as a sample and assume the causal structure is fixed on the same day, but may change across different days. With the proposed method, we found that the causal relations between these six regions can be divided into 3 groups, and the corresponding three causal structures are given in Figure \ref{fig:fMRI}.

From Figure \ref{fig:fMRI}, one can see that the three causal structures have some specific causal relationships. In the first group, the edges Sub $\to$ CA3 $\to$ PHC, CA1 $\to$ CA3 and CA1 $\to$ PHC are activated; in the second group, the edges CA3 $\to$ PRC, CA1 $\to$ Sub and EC $\to$ CA3 are activated; in the third group, the edge Sub $\to$ CA1 is also activated, while the edges EC $\to$ CA3 is inhibited. The edge CA1 $\to$ PRC exists in both groups, but with slightly different causal strengths.

\begin{figure}[ht]
    \centering
   \subfigure[Group 1]{
      \includegraphics[width=0.29\linewidth]{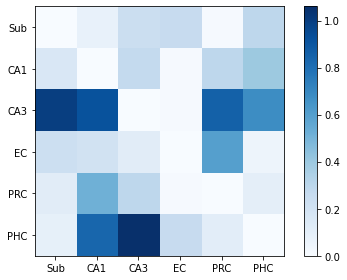}
      \label{fig:fMRI_group1}
    }
    \subfigure[Group 2]{
      \includegraphics[width=0.29\linewidth]{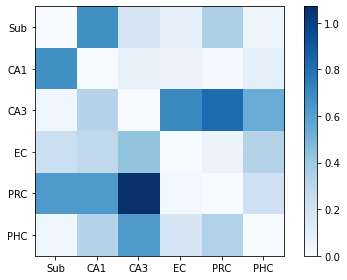}
      \label{fig:fMRI_group2}
    }
    \subfigure[Group 3]{
      \includegraphics[width=0.29\linewidth]{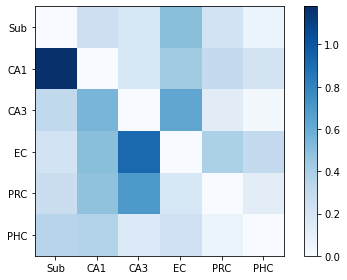}
      \label{fig:fMRI_group3}
    }
	\caption{Three causal structures learned by CCSL for three clustered groups on fMRI data.}
    \label{fig:fMRI}
\end{figure}

\paragraph{\textbf{Sachs data}}
We also apply our method to Sachs dataset \cite{sachs2005causal}. Sachs data contains many cellular protein concentrations in single cells, which are collected from varying interventions. Under different interventions, the samples don't follow the same probability distribution. Because The intervention conditions and other exogenous variables differ between those experiments. So the data sets collected under different interventions can be regarded as that from different environments \cite{huang2019specific,nagarajan2013bayesian}.In this experiment, we use the datasets collected under condition \texttt{cd3cd28+U0126} and under condition \texttt{cd3cd28+aktinhib}. The datasets from different intervention conditions are regarded as different groups. According to each dataset of a subject containing 30 samples, we divide the data into multiple subsets. Thus, each group includes 25 subjects, and the data of a subject contains 11 measured variables and 30 sample sizes. 

\begin{table}[ht]
	\centering
	\caption{Clustering performance on sachs data.}
	\begin{tabular}{|c|c|p{0.9cm}| p{1.72cm}|p{1cm}|p{1cm}|}
		\hline
		Method & CCSL & Kmeans (DTW) & Kmeans (Euclidean) & DBSCAN (DTW) & OPTICS (DTW) \\ \hline
		ARI & \textbf{0.91} & 0.84 & 0.69 & 0.03 & 0.34 \\ \hline
	\end{tabular}
	\label{table:Real_sachs}
\end{table}

After applying our method to those datasets, we find these subjects can be divided into two groups whose learned causal structures are given in Figure \ref{fig:sachs}. Table \ref{table:Real_sachs} shows the clustering performance on Sachs data. CCSL obtain the highest ARI among all the compared methods. DBSCAN (DTW) receives the worst result. This is because some interventions change a few causal relationships, which are not easy to capture the different only measuring the distance in the term of correlation. Causal relationships are indeed to consider for clustering. 

From Figure \ref{fig:sachs}, one can see that most of the causal relationships exist in both groups, while few causal relationships are different. In condition 1, reagents CD3, CD28 and U0126 are used. Then, in group 1, there exist Erk $\to$ Raf $\to$ Mek, which is consistent with the case that Mek is inhibited by Erk. In condition 2, reagents CD3, CD28 and akt-inhibitor are used to directly inhibit Akt. So we can find that the cause of Akt changes.  

\begin{figure}[ht]
    \centering
    \subfigure[The causal structure of group 1]{
		\includegraphics[width=0.45\linewidth]{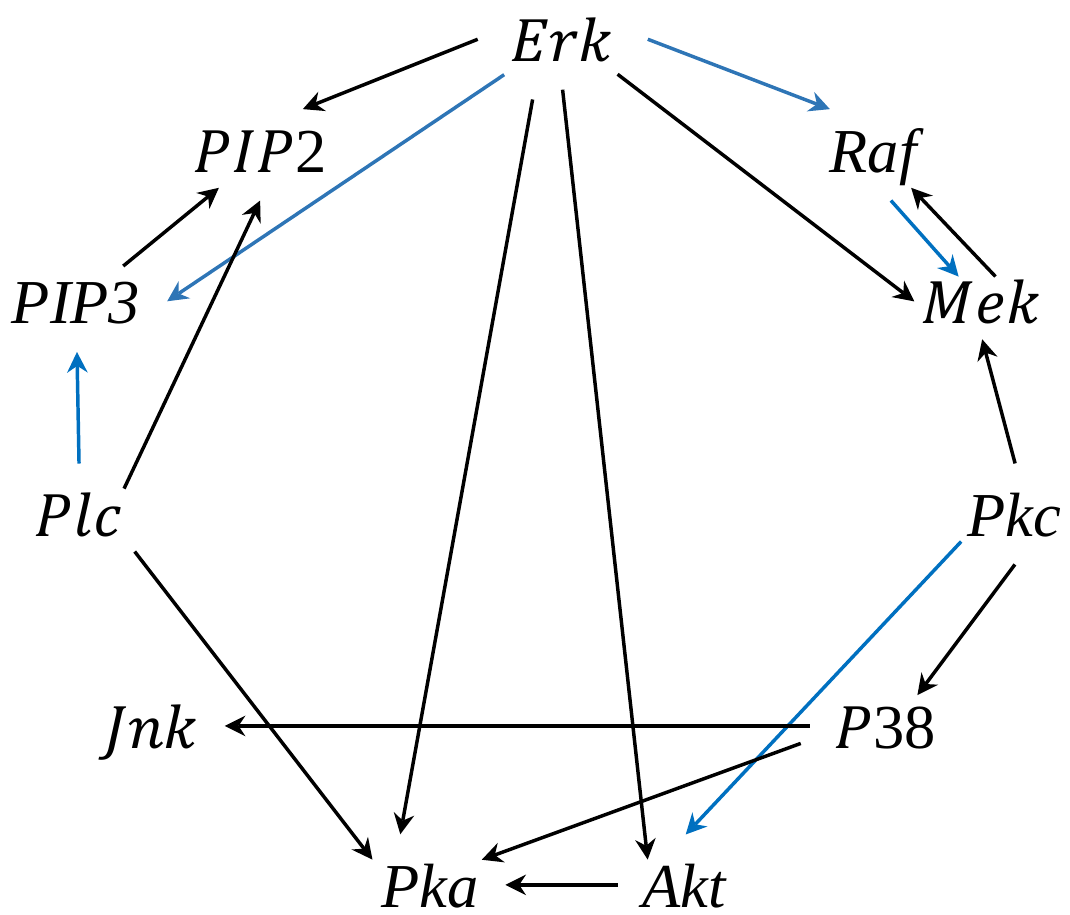}
		\label{fig:sachs_group1}
	}
	\subfigure[The causal structure of group 2]{
		\includegraphics[width=0.45\linewidth]{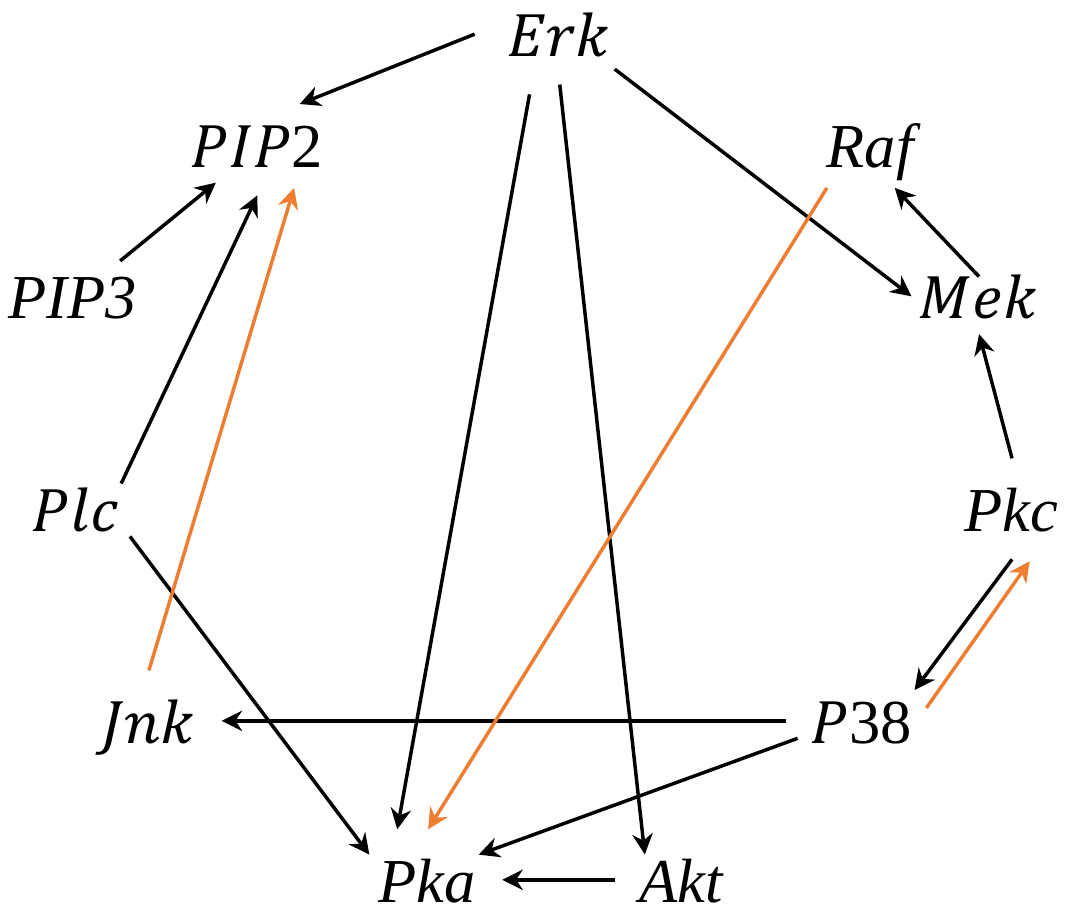}  
		\label{fig:sachs_group2}
	}
	\caption{Two causal structures learned by CCSL for two clustered groups on Sachs data. Blue line in (a) denotes the specified causal edges for group 1. Oranges line in (b) denotes the specified causal edges for group 2. Black line in both sub-figures denotes the common causal edges in both groups.}
	\label{fig:sachs}
\end{figure}


\section{Conclusion}
In this paper, we propose a Clustering and Causal Structure Learning (CCSL) method to cluster samples that share a causal mechanism and to learn their causal structures. We cluster observed samples into different groups using the proposed causality-related Chinese Restaurant Process based on causal mechanisms. At the same time, we estimate the common causal relationships among variables for each group. The combination of these steps leads to an objective function, and subsequently, an optimization solution is provided to solve this problem. The experiments both on synthesis data and real-world data evaluate the correctness and effectiveness of our method in terms of clustering and causal structure learning. Not only has this work provided a causal discovery method for non-i.i.d data, but it also illustrates that the similarity of causal structures can be more effective in clustering samples generated by the same causal mechanism than the similarity of correlation.  

In this work, we only show the identification of the causal clustering structural model under the linear non-Gaussian assumption, but this framework can also be applied to the discrete case, nonlinear case, and so on. For example, with the help of a score-based method or causal functional model, the likelihood function of discrete or nonlinear data can be easily incorporated into Eq. (\ref{eq:causalCRP}). Moreover, the performance of the CCSL model can be improved with the help of prior knowledge of the causal mechanism.


\ifCLASSOPTIONcaptionsoff
  \newpage
\fi



%
%
%
\bibliographystyle{plain}
\bibliography{CCSL}
%

%
%
%




\begin{IEEEbiography}
[{\includegraphics[width=1in, height=1.25in, clip, keepaspectratio]{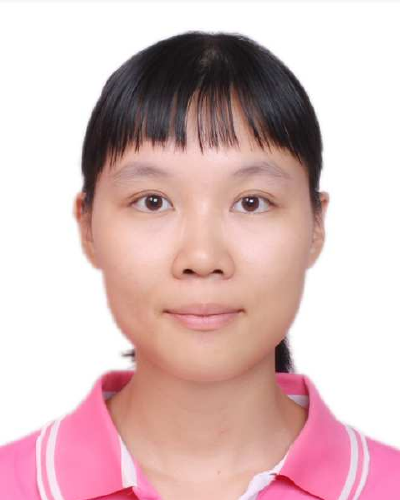}}]{Wei Chen} received the B.S. degree in computer science and the Ph.D. degree in computer application engineering from the Guangdong University of Technology, Guangzhou, China, in 2015 and 2020,
respectively.

She is currently a postdoctoral researcher at the School of Computer, Guangdong University of Technology. She was a visiting student at Carnegie Mellon University, Pittsburgh, PA, USA, from 2018 to 2019. Her research interests include causal discovery and its applications.
\end{IEEEbiography}
\vspace{-8ex}

\begin{IEEEbiography}
[{\includegraphics[width=1in, height=1.25in, clip, keepaspectratio]{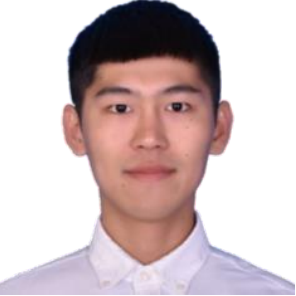}}]{Yunjin Wu} received his B.Eng degree in Internet of Things Engineering in 2019. He is currently pursuing the M.S. degree with the School of Computer, Guangdong University of Technology. 

His research interests include machine learning and its applications.
\end{IEEEbiography}
\vspace{-8ex}

\begin{IEEEbiography}[{\includegraphics[width=1in, height=1.25in, clip, keepaspectratio]{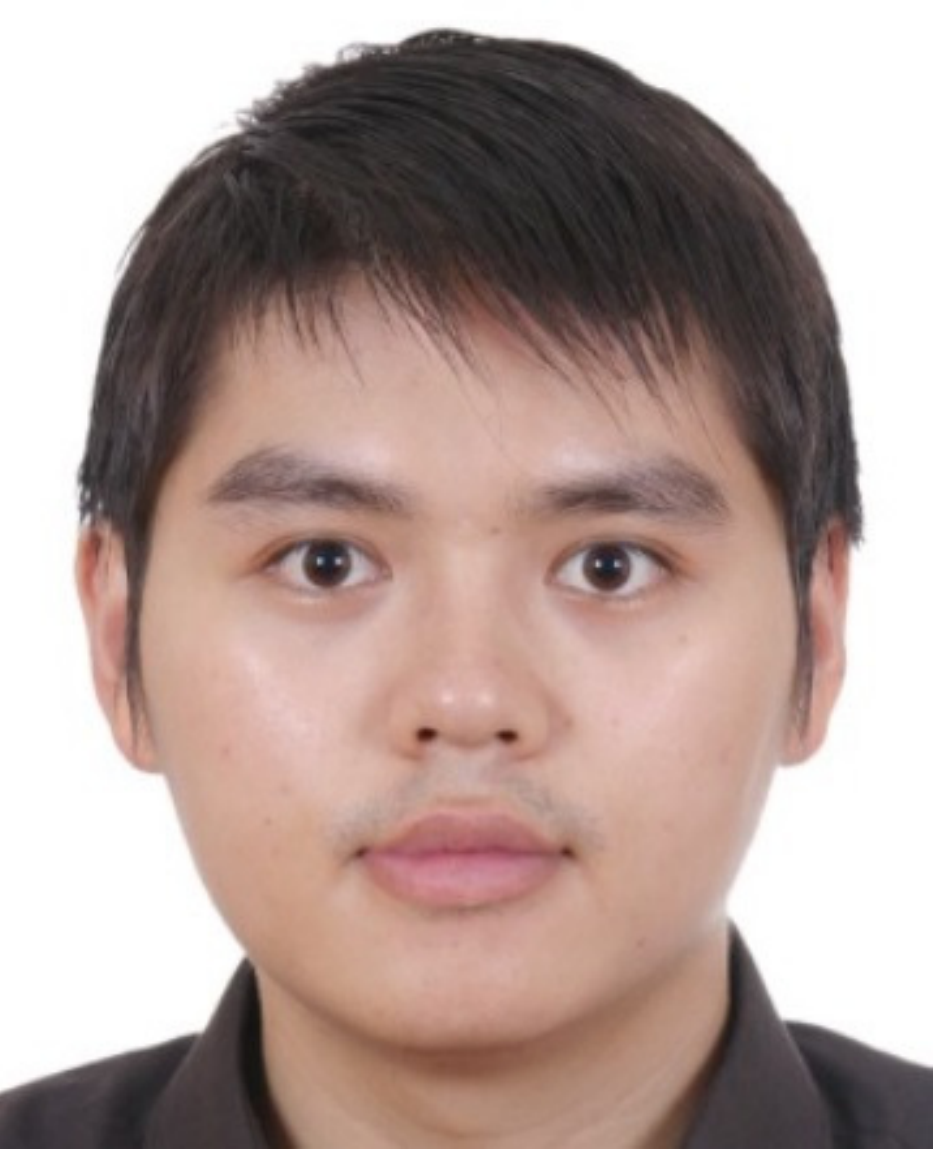}}]{Ruichu Cai} is currently a professor in the school of computer science and the director of the data mining and information retrieval laboratory, Guangdong University of Technology. He received his B.S. degree in applied mathematics and Ph.D. degree in computer science from South China University of Technology in 2005 and 2010, respectively. 
 
 His research interests cover various topics, including causality, deep learning, and their applications. He was a recipient of the National Science Fund for Excellent Young Scholars, the Natural Science Award of Guangdong, and so on awards. He has served as the area chair of ICML 2022, NeurIPS 2022, and UAI 2022, senior PC for AAAI 2019-2022, IJCAI 2019-2022, and so on. He is now a senior member of CCF and IEEE. 
\end{IEEEbiography}
\vspace{-8ex}

\begin{IEEEbiography}[{\includegraphics[width=1in, height=1.25in, clip, keepaspectratio]{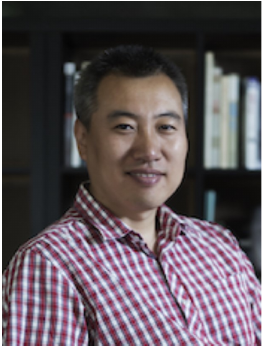}}]{Yueguo Chen} received the BS and master's degrees in mechanical engineering and control engineering from Tsinghua University, Beijing, in 2001 and 2004. He received the PhD degree in computer science from the National University of Singapore in 2009. He is currently an associate professor in the Key Laboratory of Data Engineering and Knowledge Engineering (MOE), Renmin University of China.
	
He was a visiting young faculty at Microsoft Research Asia in 2010 and 2014, and a senior visting scientist at University of Illinois Urbana-Champaign in 2017. His research interests include interactive analysis systems of big data and semantic search.
\end{IEEEbiography}
\vspace{-8ex}

\begin{IEEEbiography}
	[{\includegraphics[width=1in, height=1.25in, clip, keepaspectratio]{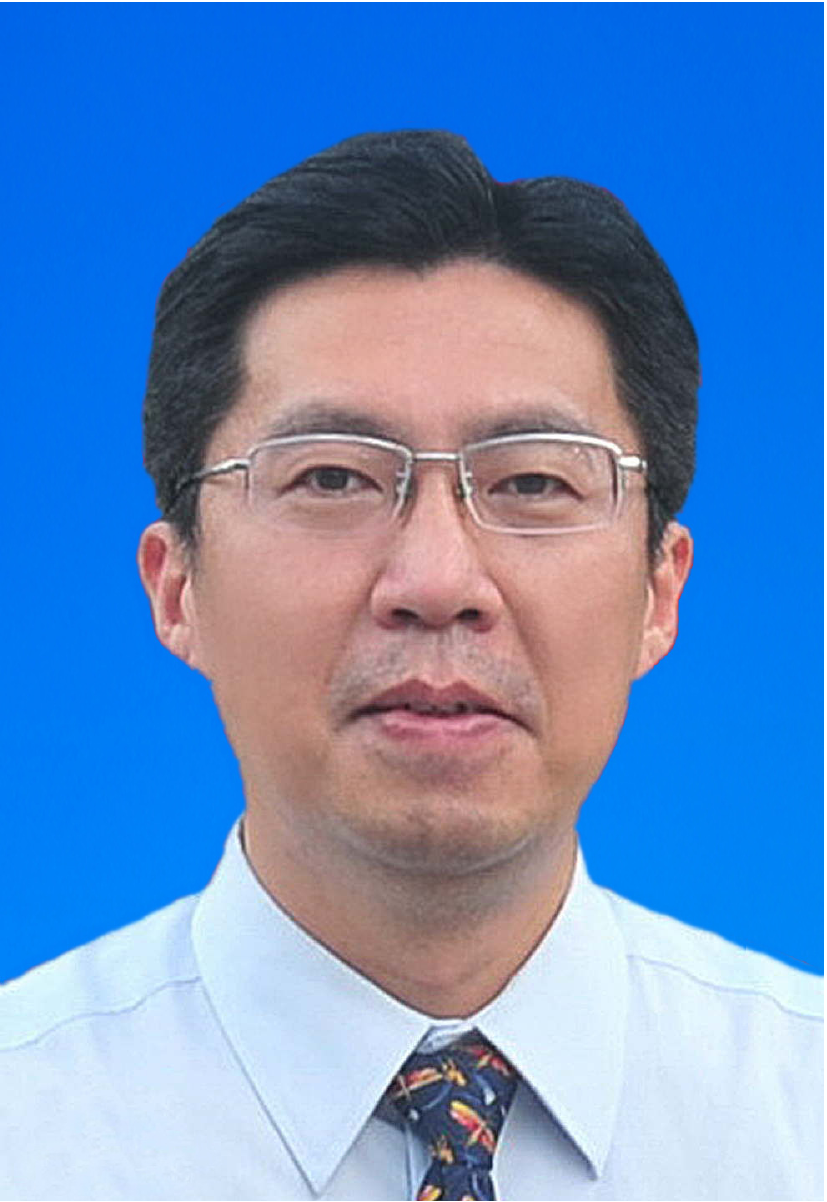}}]{Zhifeng Hao} received his B.S. degree in Mathematics from the Sun Yat-Sen University in 1990, and his Ph.D. degree in Mathematics from Nanjing University in 1995. He is currently a Professor in the School of Computer, Guangdong University of Technology, and College of Science, Shantou University. 
	
	His research interests involve various aspects of Algebra, Machine Learning, Data Mining, Evolutionary Algorithms.
\end{IEEEbiography}

\end{document}